\documentclass[10pt,twocolumn,letterpaper]{article}

\usepackage[pagenumbers]{cvpr} 

\definecolor{cvprblue}{rgb}{0.21,0.49,0.74}
\usepackage[pagebackref,breaklinks,colorlinks,allcolors=cvprblue]{hyperref}

\usepackage{url}
\usepackage{booktabs} 
\usepackage{tabularray}
\usepackage{amsthm,amsmath,amssymb,algpseudocode,algorithm,stmaryrd}
\usepackage{wrapfig}
\definecolor{forestgreen}{RGB}{12, 200, 12}%
\definecolor{forestred}{RGB}{202,12,22}%
\definecolor{darkgreen}{RGB}{51, 153, 102}

\usepackage{adjustbox}
\usepackage{subcaption} 
\usepackage{multirow}
\usepackage{newtxtext, newtxmath} 
\usepackage{enumitem}
\newtheorem{assumption}{Assumption}
\newtheorem{lemma}{Lemma}
\newtheorem{theorem}{Theorem}

\def\paperID{30961} 
\def\confName{CVPR}
\def\confYear{2026}

\title{A Faster Path to Continual Learning}

\author{
    Wei Li$^{1}$ \quad
    Hangjie Yuan$^{2}$ \quad
    Zixiang Zhao$^{3}$ \quad
    Borui Kang$^{4}$ \quad
    Ziwei Liu$^{1,5}$ \quad 
    Tao Feng$^{6}$\thanks{Corresponding Author: Tao Feng (fengtao.hi@gmail.com)} \\
    $^1$College of Computer Science, Sichuan University, China \\
    $^2$College of Computer Science and Technology, Zhejiang University, China\\
    $^3$Photogrammetry and Remote Sensing Lab, ETH Z\"urich, Switzerland \\ 
    $^4$School of Computer Science, Nanjing University, China \\
    $^5$College of Computing and Data Science, Nanyang Technological University, Singapore \\
    $^6$Department of Computer Science and Technology, Tsinghua University, China \\
    {\tt\small \{ymjiii98, fengtao.hi\}@gmail.com, hj.yuan@zju.edu.cn} 
}

\begin{document}
\maketitle
\begin{abstract}
Continual Learning (CL) aims to train neural networks on a dynamic stream of tasks without forgetting previously learned knowledge. Among optimization-based approaches, C-Flat has emerged as a promising solution due to its plug-and-play nature and its ability to encourage uniformly low-loss regions for both new and old tasks. However, C-Flat requires three additional gradient computations per iteration, imposing substantial overhead on the optimization process.
In this work, we propose C-Flat Turbo, a faster yet stronger optimizer that significantly reduces the training cost. We show that the gradients associated with first-order flatness contain direction-invariant components relative to the proxy-model gradients, enabling us to skip redundant gradient computations in the perturbed ascent steps. Moreover, we observe that these flatness-promoting gradients progressively stabilize across tasks, which motivates a linear scheduling strategy with an adaptive trigger to allocate larger turbo steps for later tasks.
Experiments show that C-Flat Turbo is 1.0× to 1.25× faster than C-Flat across a wide range of CL methods, while achieving comparable or even improved accuracy.
\end{abstract}
    
\section{Introduction}
\label{sec:intro}

In the open world, learning systems are expected to absorb new knowledge incrementally, a process known as continual learning (CL)~\cite{zhou2023class, wang2024comprehensive, qiu2025train}. A major obstacle in CL is catastrophic forgetting. To address this challenge, various approaches have been proposed~\cite{zhou2024continual, wang2022learning, smith2023coda}, including memory-based, regularization-based, expansion-based, and pre-trained model (PTM)-based methods~\cite{feng2025zeroflow,hadsell2020embracing,liu2026continual}. Among them, PTM-based approaches~\cite{jung2023generating, tang2023prompt, khan2023introducing} have demonstrated particularly strong resistance to forgetting due to their superior generalization capabilities.

Beyond architectural and data-level strategies, the optimization process of CL has recently received increasing attention~\cite{jia2022visual, tang2023prompt, khan2023introducing, khattak2023self}. A line of work shows that sharpness-aware minimization serves as a powerful mechanism for improving generalization and mitigating forgetting~\cite{he2019asymmetric, foret2020sharpness, zhong2022improving, zhuang2022surrogate}. For instance, Mehta et al.~\cite{DBLP:journals/jmlr/MehtaPCS23} provide theoretical and empirical evidence that optimizing for flat regions helps reduce forgetting. Following this insight, C-Flat~\cite{bian2024make} was recently proposed as a CL-friendly optimizer that improves stability across the joint knowledge space of new and old tasks~\cite{deng2021flattening, shi2021overcoming}. By explicitly seeking uniformly flat minima, C-Flat effectively alleviates forgetting caused by distribution and task shifts~\cite{shi2021overcoming, kong2023overcoming, zhuang2022surrogate}.
However, the flatness alignment mechanism of C-Flat incurs significant computational overhead. (i) Its zeroth-order sharpness perturbation requires an additional backward pass, doubling the cost~\cite{zhong2022improving, zhang2023gradient, liu2022towards}. (ii) Its first-order flatness term, based on Gradient Norm Aware Minimization (GAM)~\cite{zhang2023gradient}, involves calculating gradient norms at both a proxy model and its perturbed state, leading to two extra backward passes~\cite{zhuang2022surrogate}. 

In this work, we observe that the orthogonal component of the first-order flatness gradient changes more gradually than both the empirical gradient and the zeroth-order sharpness term, similar to the behavior observed in LookSAM~\cite{liu2022towards}, which focuses only on zeroth-order sharpness. This stability inspires us to periodically skip redundant SAM and GAM computations and take shortcuts along these stable directions, thereby maintaining C-Flat's effectiveness while reducing computational cost. Moreover, we observe that both sharpness and flatness gradients decrease not only as training progresses within each task, but also across tasks, indicating an overall trend toward increased stability. Motivated by this observation, (i) we introduce a stage-wise turbo-step scheduler that flexibly adjusts the shortcut frequency, and (ii) we integrate an adaptive policy that dynamically combines C-Flat with SGD for further efficiency gains.

Our technical contributions are three-fold:

(i) We identify a direction-invariant component in the first-order flatness gradients and propose C-Flat Turbo, an efficient variant of C-Flat that selectively takes shortcuts along these stable directions to reach flatter regions with substantially lower cost.

(ii) We reveal a stabilization trend in sharpness- and flatness-aware gradients during continual learning, and based on this, introduce a stage-wise linear turbo-step scheduler together with an adaptive triggering mechanism that dynamically regulates when C-Flat regularization should be applied.

(iii) Through extensive experiments, we demonstrate that C-Flat Turbo consistently matches or surpasses the performance of state-of-the-art CL methods, while being up to 1.25× faster than standard C-Flat.

\section{Related Work}

\textbf{Continual learning.} Along this line, continual learning can be broadly categorized into three groups~\cite{wang2024comprehensive,hadsell2020embracing} as follows, \textit{(i) Memory-based methods} maintain a limited memory budget to resist forgetting~\cite{rebuffi2017icarl,NEURIPS2019_fa7cdfad,jeeveswaran2023birt,sun2023regularizing}. Many efforts selectively store a few representative exemplars for rehearsal during CL. Apart from direct replay, some other efforts~\cite{deng2021flattening,lin2022trgp,saha2020gradient,lin2023pcr,sun2023decoupling} use proxies of previous knowledge as memory to overcome forgetting, such as gradient-space bases and representative prototypes. \textit{(ii) Regularization-based methods} are characterized by introducing favorable regularization terms to trade off new and old knowledge~\cite{DBLP:journals/pami/LiH18a,kirkpatrick2017overcoming,cha2020cpr}. Common practices include weight~\cite{rudner2022continual,kim2022warping,akyurek2021subspace}, function~\cite{DBLP:journals/pami/LiH18a,oh2022alife}, and feature regularization~\cite{bhat2023task, gao2022r}, which encourage these spaces to remain close to their original states. Like natural cognitive systems, this consolidation helps preserve parameters or representations that are important for previous tasks. \textit{(iii) Expansion-based methods} aim to dynamically modularize network structures towards each task to tackle forgetting~\cite{zhou2022model}. Methods in this group construct task-specific parameters or architecture (e.g., parameter allocation~\cite{liu2021adaptive,abati2020conditional}, model decomposition/mixture, and modular network~\cite{DBLP:conf/cvpr/YanX021,zhu2022self}) to explicitly reduce inter-task interference~\cite{DBLP:conf/icml/SerraSMK18,hu2023dense,DBLP:conf/iclr/YoonKYH20}. In other words, these efforts shift the burden of storing raw data to the retention architecture~\cite{zhou2022model}.

\noindent\textbf{Continual learning using generalization.} The strong generalizability of PTMs further advances the CL~\cite{zhou2024continual}. Following ~\cite{zhou2024continual}, we categorize these studies into three groups as follows,
\textit{(i) Prompt-based methods} leverage prompt learning to enable lightweight updates to PTM~\cite{jia2022visual}. Many efforts~\cite{wang2022learning,smith2023coda,jung2023generating,tang2023prompt,feng2022overcoming} resort to various prompt selections, e.g., the way of prompt retrieval, task-special prompts, and prompt generation. Moreover, some work extends the concept of prompts to broader scenes~\cite{gan2023decorate,khattak2023self,khan2023introducing}, i.e., visual prompts, pre-trained vision-language model prompts, etc. Overall, these efforts strike a balance between the generalizability of PTM and the encoding of information from downstream tasks.
\textit{(ii) Representation-based methods} focus on building classifiers by leveraging the generalization capability of PTM~\cite{zhou2023revisiting}. Briefly, this line of work mainly focuses on calibrating classifiers and optimizing their internal relationships~\cite{zhou2024expandable,mcdonnell2024ranpac,zhang2023slca,zhu2021prototype}. For example, these methods concatenate backbone features to improve classifier representations, use random projections to remove class-wise correlations, and replay features to calibrate the classifier.
\textit{(iii) Model mixture-based methods} introduce a set of models and utilize various model-mixture approaches, such as model ensembling and model merging, to generate final outputs~\cite{wang2023isolation,gao2023unified,zhou2023learning,wang2022s}. In general, methods in this category integrate several PTMs with strong generalization capabilities to produce more robust results~\cite{zhou2024continual}. Yet, some of the drawbacks that arise are associated with these efforts, e.g. extra computational overhead and memory buffer.

\noindent\textbf{Better and more efficient generalization in continual learning.} In continual learning (CL), research on how flat minima~\cite{he2019asymmetric,foret2020sharpness,baldassi2020shaping} affect catastrophic forgetting is still in its early stages. A few studies have explored flatness evaluation during CL in specific scenarios~\cite{deng2021flattening,shi2021overcoming}, such as tuning parameters within flat minima regions. Notably, C-Flat~\cite{bian2024make}, a CL-tailored optimizer, introduces the concept of continual flatness to mitigate forgetting, opening a new avenue in CL research. While these efforts provide valuable insights into the role of flat minima in CL, they often face challenges in optimization efficiency, particularly due to the entrapment in random perturbations~\cite{liu2022towards,du2022sharpness,du2021efficient}, especially when dealing with successive tasks. Moreover, most sharpness-aware CL methods directly inherit the full computational cost of SAM updates, leading to multiple additional backward passes compared to vanilla optimizers. This creates a tension between pursuing flatter minima and maintaining practical training efficiency, which is particularly critical in long task sequences and large-scale PTM-based CL. In this paper, we present an efficient version of C-Flat that improves the performance of CL while maintaining optimization efficiency.

\section{Method}
\label{sec:method}

\subsection{Rethinking the Mechanism of C-Flat}
\label{subsec:C-Flat}

\begin{figure}
    \centering
    \includegraphics[width=0.45\textwidth]{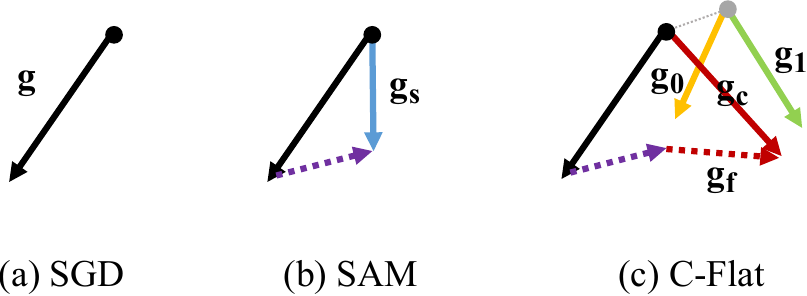}
    \caption{Brief illustration of C-Flat~\cite{bian2024make}. (a) SGD optimizes along the negative direction of the gradients, $g = \nabla L(f(\theta^T))$. (b) SAM~\cite{foret2020sharpness} computes the gradients $g_s$ at an adversarially perturbed position $\theta + \rho \cdot g/\|g\|$, and then updates the original model parameters. (c) C-Flat~\cite{bian2024make} further calculates the first-order flatness gradient $g_f$, based on the perturbed parameters $\theta + \rho \cdot (g_s - g)/\|(g_s - g)\|$.}
    \vspace{-2mm}
    \label{fig:cflat}
\end{figure}

Let $\mathcal{S}^t = \{(\mathbf{x}_i^t, \mathbf{y}_i^t)\}_{i=1}^{n^t}$ denote the training set with $n^t$ samples for task $t$, and let $\ell(f(\mathbf{x};\boldsymbol{\theta}),\mathbf{y})$ be the per-sample loss of the neural network $f$ with parameters $\boldsymbol{\theta}$ on a data point $(\mathbf{x},\mathbf{y})$. Continual Learning (CL) aims to learn a model $f$ with parameters $\boldsymbol{\theta} \in \mathbb{R}^d$ that minimizes the statistical risk across all tasks seen up to the current task $T$, under the constraint of limited or no access to previous data $\mathcal{S}^t$ for $t < T$. Specifically, CL methods optimize the model parameters $\boldsymbol{\theta}_T$ during training on task $T$ by minimizing the empirical loss over the available data: $\boldsymbol{\theta}_T = \arg\min_{\boldsymbol{\theta}} \mathcal{L}(\boldsymbol{\theta}, \mathcal{S}^T)$, where $\mathcal{L}(\boldsymbol{\theta}, \mathcal{S}^T) = \frac{1}{n^T} \sum_{i=1}^{n^T} \ell(f(\mathbf{x}_i^T;\boldsymbol{\theta}), \mathbf{y}_i^T)$. Depending on the specific CL methods, the training data may include a mixture of current task data $\mathcal{S}^T$ and additional data, such as stored exemplars or reconstructed samples from previous tasks. For simplicity, we omit $\mathcal{S}^T$ and denote $\boldsymbol{\theta}$ as $\boldsymbol{\theta}^T$ later.

As shown in Figure~\ref{fig:cflat}, the C-Flat optimizer~\cite{bian2024make} mitigates catastrophic forgetting by jointly optimizing for zeroth-order sharpness $\mathcal{R}_{\rho}^0(\boldsymbol{\theta})$ and first-order flatness $\mathcal{R}_{\rho}^1(\boldsymbol{\theta})$, encouraging the model to converge to parameter regions with uniformly low loss and curvature. The optimization objective on task $T$ can be written as
\begin{equation}
\label{eq:total}
\min_{\boldsymbol{\theta}} \; \mathcal{L}(\boldsymbol{\theta})
+ \mathcal{R}_{\rho}^0(\boldsymbol{\theta})
+ \lambda \cdot \mathcal{R}_{\rho}^1(\boldsymbol{\theta}),
\end{equation}
where $\lambda > 0$ is a balancing hyperparameter and $\rho > 0$ defines the neighborhood radius around the current parameter $\boldsymbol{\theta}$. Specifically, the zeroth-order term is defined as
\begin{equation}
    \mathcal{R}_{\rho}^0(\boldsymbol{\theta})
    = \max_{\|\boldsymbol{\epsilon}_0\| \leq \rho}
      \mathcal{L}(\boldsymbol{\theta} + \boldsymbol{\epsilon}_0)
      - \mathcal{L}(\boldsymbol{\theta}),
\end{equation}
which encourages uniformly low loss within a local neighborhood of $\boldsymbol{\theta}$. The first-order flatness term is given by
\begin{equation}
\mathcal{R}_{\rho}^1(\boldsymbol{\theta})
= \rho \cdot \max_{\|\boldsymbol{\epsilon}_1\| \leq \rho}
\Big\| \nabla \mathcal{L}(\boldsymbol{\theta} + \boldsymbol{\epsilon}_1) \Big\|,
\end{equation}
which encourages low curvature of the loss landscape. Here, $\|\cdot\|$ denotes the $\ell_2$ norm.

\noindent\textbf{Propagation flow.} 
Assuming that the loss function $\mathcal{L}$ is differentiable and bounded, and following the derivation of $\nabla R_{\rho}^0(\theta)$ in~\cite{foret2020sharpness, singh2020woodfisher, zhang2023gradient}, we can approximate $\nabla R_{\rho}^1(\theta)$ as follows. More notations are provided in Appendix~\ref{subsec:derivation}.

\begin{equation}
\label{eq:grad_s}
\nabla \mathcal{R}_{\rho}^0(\boldsymbol{\theta}) 
= \nabla \mathcal{L}(\boldsymbol{\theta} + \boldsymbol{\epsilon}_0^*) - \nabla \mathcal{L}(\theta),
\quad 
\boldsymbol{\epsilon}_0^* \approx \rho \cdot 
\frac{\nabla \mathcal{L}(\boldsymbol{\theta})}
{\big\| \nabla \mathcal{L}(\boldsymbol{\theta}) \big\|}.
\end{equation}

{\setlength{\belowdisplayskip}{6pt}
\begin{align}
\label{eq:grad_f}
    \nabla \mathcal{R}^1_{\rho}\!\left(\boldsymbol{\theta}\right) 
    &\approx \rho \cdot \nabla 
        \left\| \nabla \mathcal{L}\!\left(\boldsymbol{\theta}+\boldsymbol{\epsilon}_1^*\right) \right\| \\ \nonumber
    &\approx \nabla \mathcal{L}\!\left(
        \boldsymbol{\theta}+\boldsymbol{\epsilon}_1^* 
        + \rho \cdot \frac{\nabla \mathcal{L}(\boldsymbol{\theta}+\boldsymbol{\epsilon}_1^*)}
                {\big\| \nabla \mathcal{L}(\boldsymbol{\theta}+\boldsymbol{\epsilon}_1^*) \big\|}\right) 
    - \nabla \mathcal{L}\!\left(\boldsymbol{\theta}+\boldsymbol{\epsilon}_1^*\right), \\ \nonumber
    \boldsymbol{\epsilon}_1^* &\approx \rho \cdot (\boldsymbol{g_s} - \boldsymbol{g})/(\| \boldsymbol{g_s} - \boldsymbol{g} \|).
\vspace{-6mm}
\end{align}
}

\noindent\textbf{Notations.}
Here, we define several variations and terminology that will frequently be used later. Detailed descriptions can be found in Appendix~\ref{sec:symbols}.
\begin{enumerate}[label={},leftmargin=0pt]
\item i) \textbf{the empirical loss term:} $\boldsymbol{g} = \nabla \mathcal{L}(\boldsymbol{\theta})$ as the original gradients derived from the vanilla optimizer.

\item ii) \textbf{the zeroth-order sharpness term:} $ \boldsymbol{g}_s 
= \nabla \mathcal{L}(\boldsymbol{\theta} + \boldsymbol{\epsilon}_0^*)$ as the perturbed SAM gradients, where $\boldsymbol{g}_s - \boldsymbol{g} = \nabla \mathcal{R}_{\rho}^0(\boldsymbol{\theta}) $ measures the sharpness of the loss landscape.\footnote{We use $\boldsymbol{g}_s - \boldsymbol{g}$ rather than $\boldsymbol{g}_{vs}$ to measure sharpness, for better alignment with the flatness term $\boldsymbol{g}_f = \boldsymbol{g}_1 - \boldsymbol{g}_0$ used later. Though it is not fully orthogonal to the gradient direction $\boldsymbol{g}$, it still captures the direction that promotes zeroth-order sharpness, and has been widely used in~\cite{zhao2022penalizing, zhang2023gradient} when combined with the vanilla gradient $\boldsymbol{g}$ to form a sharpness-aware update direction.}

\item iii) \textbf{the first-order flatness term:} 
$\boldsymbol{g}_f = \nabla \mathcal{R}^1_{\rho}(\boldsymbol{\theta}) = \boldsymbol{g}_1 - \boldsymbol{g}_0$ as the regularization gradients, which quantifies the flatness of the loss landscape. Here, the intermediate gradients are given by $\boldsymbol{g}_0 = \nabla \mathcal{L}(\boldsymbol{\theta} + \boldsymbol{\epsilon}_1^*)$ and $\boldsymbol{g}_1 = \nabla \mathcal{L}\left(\boldsymbol{\theta} + \boldsymbol{\epsilon}_1^* + \rho \cdot \frac{\nabla \mathcal{L}(\boldsymbol{\theta} + \boldsymbol{\epsilon}_1^*)}{\left\| \nabla \mathcal{L}(\boldsymbol{\theta} + \boldsymbol{\epsilon}_1^*) \right\|} \right)$.
\end{enumerate}

As shown above, optimizing the gradients of the C-Flat objective can be decomposed into three components: \textbf{\textit{the empirical loss term $\boldsymbol{g}$, the sharpness term $\boldsymbol{g}{s}-\boldsymbol{g}$, and the flatness term $\boldsymbol{g}{f}$}}.
Among them, $\boldsymbol{g}$ is required at every update step to reduce the empirical risk, whereas computing $\boldsymbol{g}_s$ and $\boldsymbol{g}_f$ incurs one and two additional backward passes, respectively, using the perturbed models $\boldsymbol{\theta} + \boldsymbol{\epsilon}_0^*$ and $\boldsymbol{\theta} + \boldsymbol{\epsilon}_1^*$. Both regularization terms help identify flatter regions of the loss landscape. Despite their performance benefits, searching for such regions is computationally expensive and imposes a substantial burden on CL systems. Therefore, developing time-efficient solutions becomes imperative.

\begin{figure*}[]
    \centering
    \includegraphics[width=\textwidth]{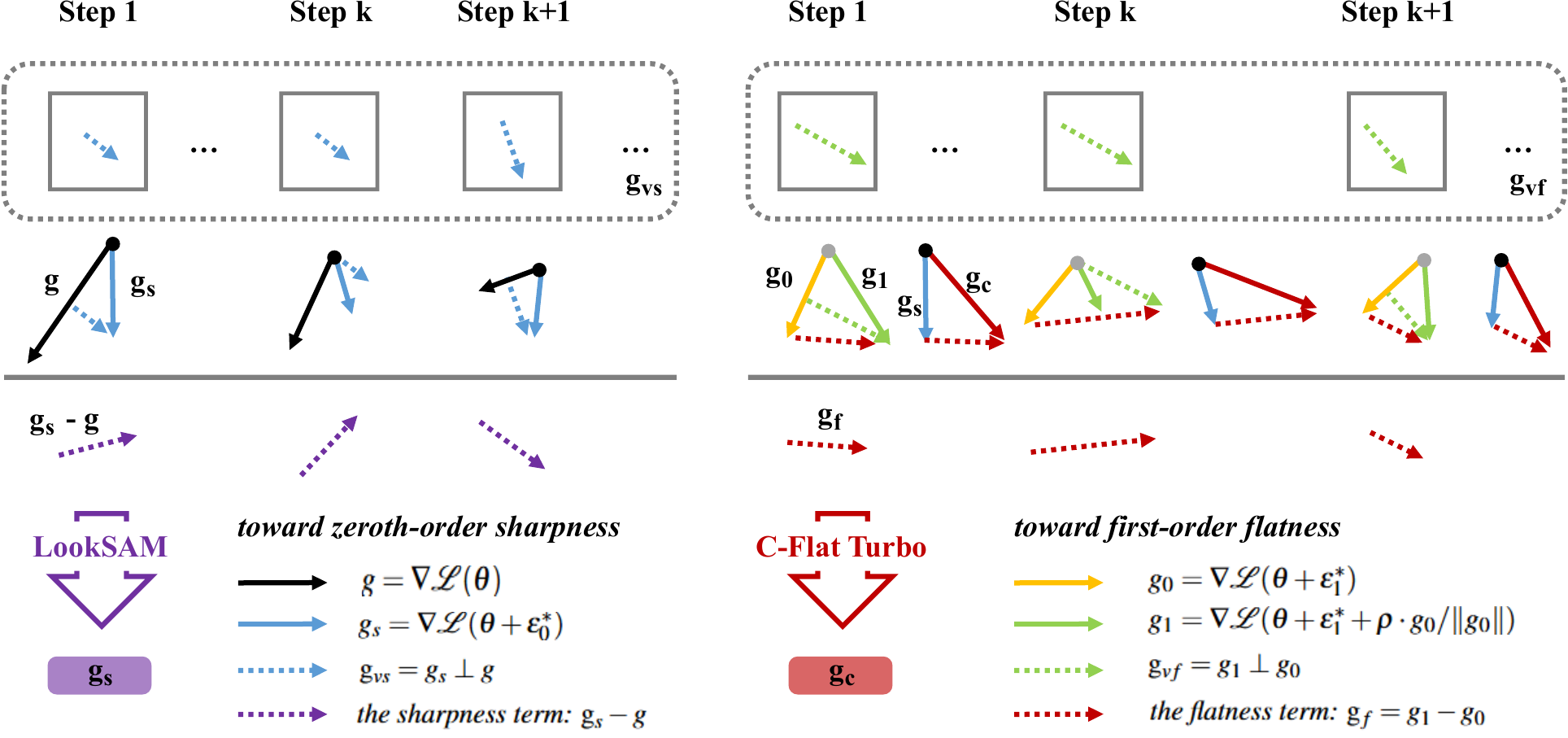}
    \caption{Schematic illustration of C-Flat Turbo. 
    \textbf{Left:} LookSAM~\cite{liu2022towards} decomposes the SAM gradients into two components: one parallel to $\boldsymbol{g}$ that reduces the empirical loss, and an orthogonal component $\boldsymbol{g}_{\mathrm{vs}}$ that guides convergence toward a common low-loss region. Empirical results show that $\boldsymbol{g}_{\mathrm{vs}}$ varies significantly more slowly than $\boldsymbol{g}$. 
    \textbf{Right:} C-Flat Turbo investigates the latent invariance of $\boldsymbol{g}_{\mathrm{vf}}$, the flatness component orthogonal to the gradients at the perturbed model $\boldsymbol{\theta} + \boldsymbol{\epsilon}_1^*$, which exhibits even slower variation than $\boldsymbol{g}_{\mathrm{vs}}$ in LookSAM.}
    \label{fig:framework}
\end{figure*}

\subsection{C-Flat Turbo}
In this section, we introduce C-Flat Turbo, an enhanced variant of C-Flat designed to accelerate training in continual learning scenarios. C-Flat Turbo exploits the observation that the orthogonal component of the first-order flatness gradient varies much more slowly than other gradient terms, allowing the optimizer to skip redundant computations associated with proxy and proxy-perturbed gradients. In addition, we incorporate an adaptive mechanism that monitors sharpness online and selectively applies C-Flat only when beneficial, falling back to vanilla optimizers otherwise. 

\subsubsection{Taking Shortcuts Toward Flatness}
\label{subsubsec:memorize}

Base optimizers like SGD and Adam reduce the loss function along gradient directions. However, the resulting model parameters are often sensitive to small perturbations or noise, which can lead to degradation of previously learned parameters when adapting to new tasks. Seeking a unified low-loss landscape within a local region around the global minima has proven effective for continual learning. C-Flat builds upon this idea by promoting flatter regions via first-order flatness, in conjunction with zeroth-order sharpness.

In the context of zeroth-order sharpness, the increment $\boldsymbol{g}_s - \boldsymbol{g}$ captures the sharpness-aware correction introduced by SAM. Following the algorithm in Algorithm~\ref{algo:opt}, we further define a direction-invariant sharpness component
\begin{equation}
\label{eq:gvs}
\boldsymbol{g}_{vs}
:= \boldsymbol{g}_s
 - \frac{\langle \boldsymbol{g}_s, \boldsymbol{g} \rangle}{\|\boldsymbol{g}\|^2}\,\boldsymbol{g},
\end{equation}
which is orthogonal to the primary gradient direction $\boldsymbol{g}$.
It has been observed that this component behaves as a more slowly varying, curvature-sensitive part of the SAM update, while facilitating the exploration of flatter regions in LookSAM~\cite{liu2022towards}. This phenomenon further motivates us to investigate whether similar direction-invariant components exist in the optimization of the first-order flatness term.

\begin{figure}[]
    \centering
    \begin{subfigure}{0.23\textwidth}
        \centering
        \includegraphics[width=\textwidth]{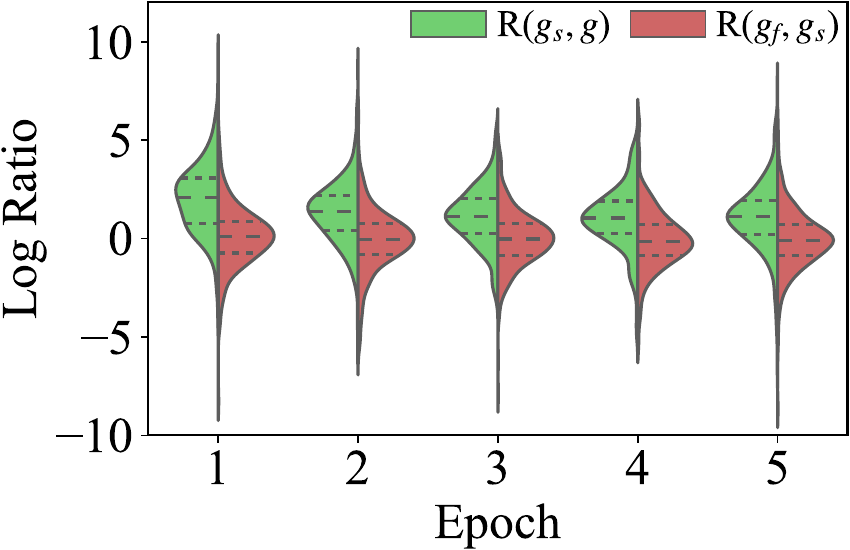}
        \caption{Gradient correction ratio.}
        \label{fig:ratio_dist}
    \end{subfigure}
    \hfil
    \begin{subfigure}{0.23\textwidth}
        \centering
        \includegraphics[width=\textwidth]{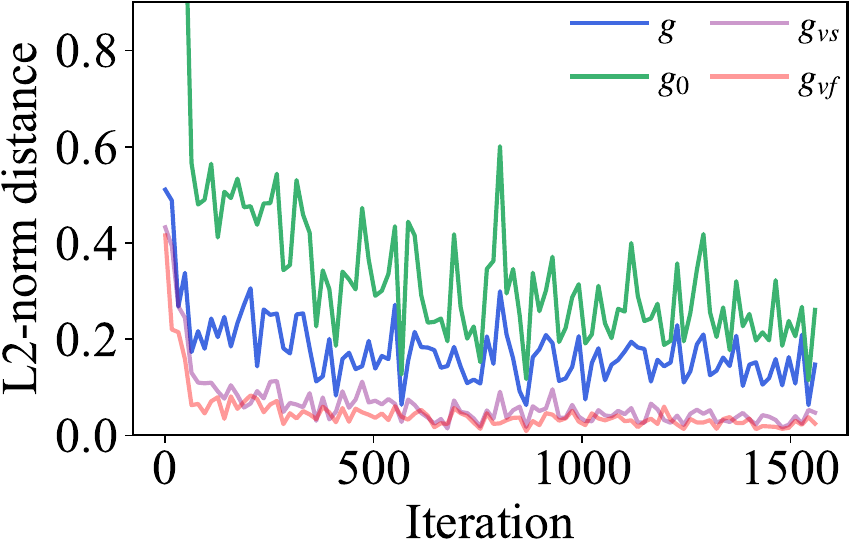}
        \caption{Training with PTM.}
        \label{fig:diff_ptm}
    \end{subfigure}
    \caption{\textbf{Left:} Distributions of gradient correction ratios for 
    $\boldsymbol{g}_s - \boldsymbol{g}$ and $\boldsymbol{g}_f$ across training epochs. 
    A larger portion of data near the distribution tails indicates increasingly pronounced differences. \textbf{Right:} L2-norm distances between gradients and their counterparts from five steps earlier. Changes along the sharpness-related direction ($\boldsymbol{g}_{vs}$) and flatness-related direction ($\boldsymbol{g}_{vf}$) evolve more slowly than those along the empirical gradient directions ($\boldsymbol{g}$ and $\boldsymbol{g}_0$).}
    \vspace{-4mm}
    \label{fig:grad_var}
\end{figure}

\begin{figure*}[t]
    \vspace{-2mm}
    \centering
    \begin{subfigure}{0.25\textwidth}
        \centering
        \includegraphics[width=\textwidth]{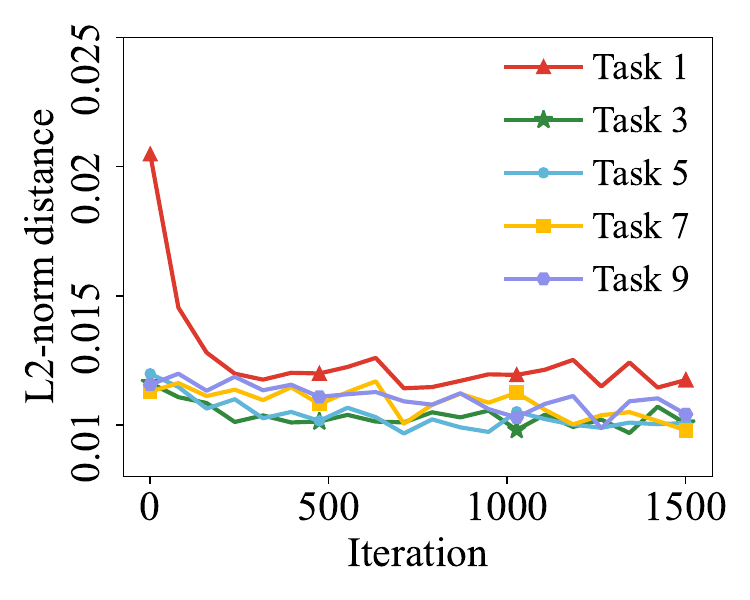}
        \caption{sharpness}
        \label{subfig:sharp_task}
    \end{subfigure}
    \hfill
    \begin{subfigure}{0.25\textwidth}
        \centering
        \includegraphics[width=\textwidth]{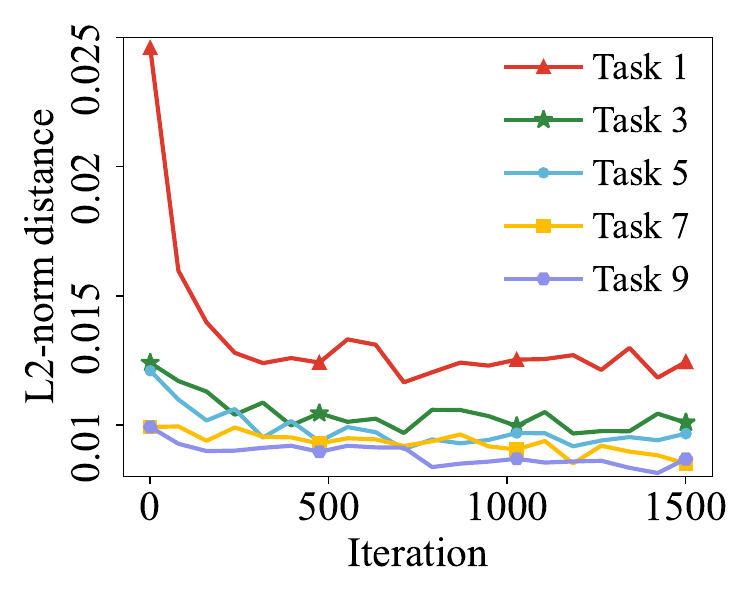}
        \caption{flatness}
        \label{subfig:flat_task}
    \end{subfigure}
    \hfill
    \begin{subfigure}{0.24\textwidth}
        \centering
        \includegraphics[width=\textwidth]{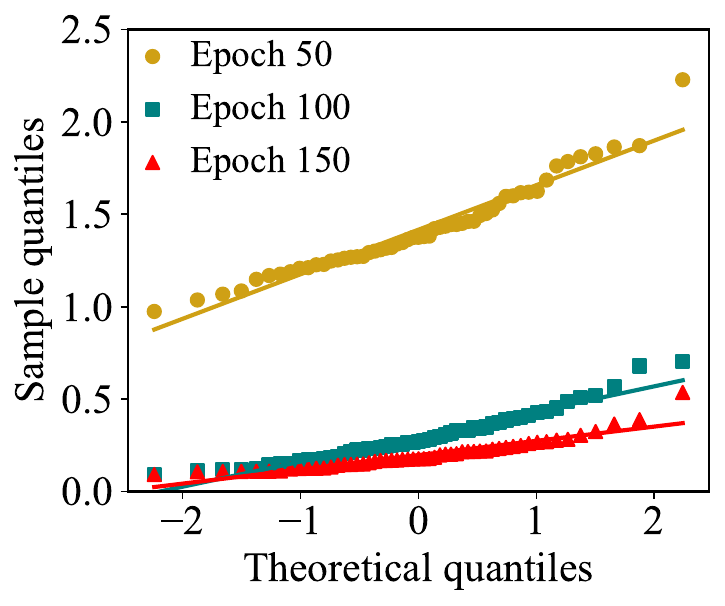}
        \caption{$\|\boldsymbol{g}\|^2$}
        \label{fig2:qq_s}
    \end{subfigure}
    \hfill
    \begin{subfigure}{0.24\textwidth}
        \centering
        \includegraphics[width=\textwidth]{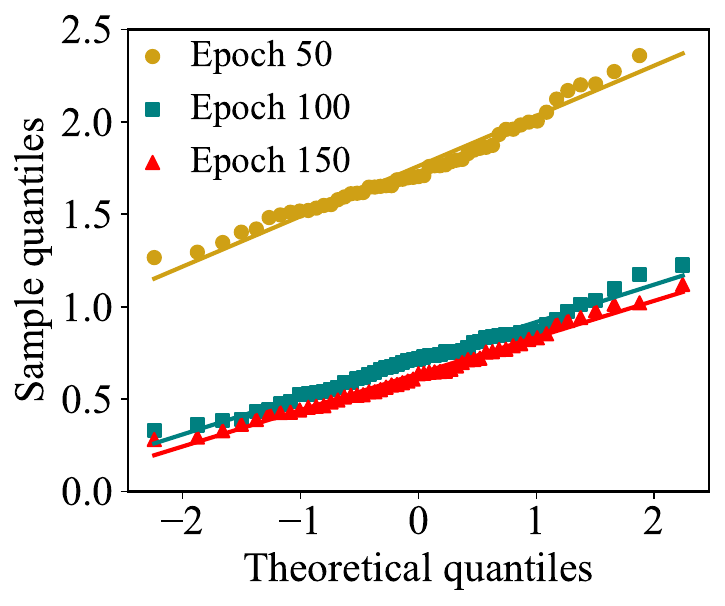}
        \caption{$\|\boldsymbol{g}_0\|^2$}
        \label{fig2:qq_f}
    \end{subfigure}
    \caption{
    \textbf{Left:} Gradient dynamics across tasks in CL. Both sharpness- and flatness-related gradients fluctuate substantially in early stages but progressively stabilize as training proceeds. 
    \textbf{Right:} Q–Q plots of $\|\boldsymbol{g}\|^2$ and $\|\boldsymbol{g}_0\|^2$, showing that both statistics gradually approach a normal distribution over the course of learning.
    }
    \vspace{-2mm}
    \label{fig:combined}
\end{figure*}

To visualize the optimization dynamics of $\boldsymbol{g}_s - \boldsymbol{g}$ and $\boldsymbol{g}_f$ relative to their reference directions, $\boldsymbol{g}$ for SAM and $\boldsymbol{g}_s$ for C-Flat, we define the gradient correction ratio as
$\text{ratio}(m, n) = \log\left(\left|\frac{m}{n + \epsilon}\right|\right), \quad \epsilon = 10^{-12}$, and then conduct experiments in EASE over 5 epochs.
Figure~\ref{fig:ratio_dist} illustrates the distributions of $(\boldsymbol{g}_s - \boldsymbol{g}, \boldsymbol{g})$ and $(\boldsymbol{g}_f, \boldsymbol{g}_s)$ in Task~1. It can be observed that $\boldsymbol{g}_s - \boldsymbol{g}$ exhibits heavier tails than $\boldsymbol{g}_f$ in the early stages, indicating stronger adjustment introduced by zeroth-order sharpness. As the optimization converges to a local minimum, sharpness and flatness jointly contribute to exploring flatter regions, with sharpness still playing a dominant role. This phenomenon suggests that $\boldsymbol{g}_f$ acts as a relatively small rectification upon SAM, even subtler than the correction $\boldsymbol{g}_s - \boldsymbol{g}$ induced by SAM itself.

In approximating $\mathcal{R}_{\rho}^1(\boldsymbol{\theta})$ in Eq.~\ref{eq:grad_f},
the perturbation $\boldsymbol{\epsilon}_1^*$ naturally emerges as a small step aligned with the sharpness term. Consequently, the point $\boldsymbol{\theta}_p = \boldsymbol{\theta} + \boldsymbol{\epsilon}_1^*$ serves as a proxy model for computing the flatness term.
As illustrated in Figure~\ref{fig:cflat}, $\boldsymbol{g}_f$ plays a role analogous to SAM, but is defined at the proxy model $\boldsymbol{\theta}_p$ located in a region of high curvature.
In this sense, we extract a direction-invariant flatness component
$\boldsymbol{g}_{vf}$ from $\boldsymbol{g}_f$ with respect to the proxy gradient $\boldsymbol{g}_0$:
\begin{equation}
\label{eq:gvf}
\boldsymbol{g}_{vf}
:= \boldsymbol{g}_f
 - \frac{\langle \boldsymbol{g}_f, \boldsymbol{g}_0 \rangle}{\|\boldsymbol{g}_0\|^2}\,\boldsymbol{g}_0,
\end{equation}
which is orthogonal to $\boldsymbol{g}_0$ and captures the flatness-promoting update direction that is invariant to the base gradient direction of the proxy model.

Figure~\ref{fig:diff_ptm} illustrates the gradient differences across iterations, measured by the L2-norm distance to the gradients from five steps earlier. We observe that $\boldsymbol{g}_{vs}$ fluctuates more slowly than $\boldsymbol{g}$, which is consistent with the findings in~\cite{liu2022towards}. Moreover, the proxy gradient $\boldsymbol{g}_0$ is noticeably more unstable than $\boldsymbol{g}$, reflecting its high-curvature nature and strong sensitivity to parameter changes. 
Interestingly, despite the considerable variation in $\boldsymbol{g}_0$, the flatness-oriented gradient $\boldsymbol{g}_f$ remains substantially more stable, and the direction-invariant component $\boldsymbol{g}_{vf}$ is even more stable than $\boldsymbol{g}_{vs}$, in line with the observations in Figure~\ref{fig:ratio_dist}.
These insights support extracting the orthogonal component $\boldsymbol{g}_{vf}$ from $\boldsymbol{g}_f$ with respect to $\boldsymbol{g}_0$ as a flatness-promoting direction, as illustrated in Figure~\ref{fig:framework}.
Consequently, for the subsequent $k-1$ surrogate steps, the cached $\boldsymbol{g}_{vf}$ can be reused to efficiently guide the search toward low-curvature directions based on the proxy model at $\boldsymbol{\theta} + \boldsymbol{\epsilon}_1^*$.
Concretely, we approximate the flatness update by adding a scaled direction-invariant term to the proxy gradient,
$\boldsymbol{g}_f \approx \boldsymbol{g}_0 + \beta \frac{\|\boldsymbol{g}_0\|}{\|\boldsymbol{g}_{vf}\|}\,\boldsymbol{g}_{vf}$,
thus avoiding recomputation of $\boldsymbol{g}_1$ when evaluating flatness at the proxy model and substantially reducing computational cost.

\subsubsection{Dynamic Control for C-Flat}
\label{subsubsec:ae-turbo}
\noindent\textbf{Stage-wise turbo step scheduling.} 
As continual learning progresses, the learned parameter space becomes increasingly flatter, with earlier classes becoming more distinctly separated. Figures~\ref{subfig:sharp_task} and~\ref{subfig:flat_task} depict the evolution of the sharpness and flatness terms over time. Both gradients fluctuate considerably during the initial incremental stage but stabilize in later stages, with this effect being particularly pronounced for flatness.
This observation motivates assigning smaller turbo steps to earlier tasks and larger intervals to later ones. To this end, we introduce a linear scheduler that adaptively adjusts the step size for computing sharpness- and flatness-aware gradients throughout training. It is simply formulated as Turbo-$k_0$ with $k_t = k_0 + 10 \cdot t / N$, where $k_0$ and $k_t$ denote the initial and task-$n$ step sizes, respectively, with a total of $N$ tasks. Empirically, C-Flat Turbo equipped with this scheduler achieves substantial speedup while maintaining competitive performance. Notably, an imprecise or even unknown value of $N$ does not pose a serious concern, since the scheduler mainly serves to gradually increase $k_t$, and small distortions of the schedule in later
tasks have little practical impact.

\noindent\textbf{Adaptive triggering of regularization.} 
Existing research on zeroth-order sharpness has proposed reducing computational overhead by combining SAM updates with standard ERM. For example, SS-SAM~\cite{zhao2022stochastic} employs a Bernoulli trial to decide whether to apply SAM, while AE-SAM~\cite{jiang2023adaptive} applies SAM only when a sharpness measure falls below a dynamically updated threshold. However, first-order flatness has received relatively little attention.
Figures~\ref{fig2:qq_s} and~\ref{fig2:qq_f} visualize the distributional properties of $\|\boldsymbol{g}\|^2$ and $\|\boldsymbol{g}_0\|^2$ using quantile–quantile (Q–Q) plots. Points that lie closer to the reference line indicate that the corresponding variable more closely follows a normal distribution.
Motivated by this observation, we use an exponential moving average (EMA) to estimate the mean and dispersion of $\|\boldsymbol{g}_0\|^2$, following~\cite{jiang2023adaptive}:
\begin{align}
\label{eq:ema_f}
\mu_{f,j} &= \delta \mu_{f, j-1} + (1 - \delta) \| \boldsymbol{g}_{0j} \|^2, \\ \nonumber
\sigma_{f,j} &= \delta \sigma_{f, j-1} + (1 - \delta) \bigl(\|\boldsymbol{g}_{0j}\|^2 - \mu_{f,j}\bigr)^2,
\end{align}
where $j$ denotes the current iteration and $\delta = 0.9$ is a decay factor that discounts outdated gradient values. The flatness regularization is triggered only when
\begin{equation}
\label{eq:trig}
    \|\boldsymbol{g}_{0j}\|^2 > \mu_{f,j} + \sigma_{f,j}.
\end{equation}
A detailed description of the overall optimization procedure is provided in Algorithm~\ref{algo:opt} in the Appendix.

\begin{table*}[h]
\caption{Accuracy and training speeds using five state-of-the-art methods, with a pre-trained ViT-B/16-IN1K backbone. \textcolor{forestred}{Red} and \textcolor{forestgreen}{green} denote the baseline and the efficient optimizer, respectively. \textbf{Bolded} indicates the best result.}
\label{tab:ptm}
\centering
\resizebox{.85\textwidth}{!}{
\begin{tabular}{c l cccc cccc c}
\toprule
\multirow{2}{*}{\centering Model} & \multirow{2}{*}{\centering Method} & \multicolumn{2}{c}{\centering CIFAR100 B0\_Inc10} & \multicolumn{2}{c}{\centering CUB B0\_Inc10} & \multicolumn{2}{c}{\centering IN-R B0\_Inc20} & \multicolumn{2}{c}{\centering ObjNet B0\_Inc10} & \multirow{2}{*}{\centering Img/s$\uparrow$} \\
& & Avg$\uparrow$ & Last$\uparrow$ & Avg$\uparrow$ & Last$\uparrow$ & Avg$\uparrow$ & Last$\uparrow$ & Avg$\uparrow$ & Last$\uparrow$ & \\
\midrule
\multirow{6}{*}{\centering \rotatebox{90}{\textit{Typical}}} 
& iCaRL~\cite{rebuffi2017icarl} & 77.83 & 66.64 & 82.91 & 74.00 & 72.13 & 61.62 & 48.06 & 28.20 & 73.35 (100\%) \\
& \textit{+C-Flat~\cite{bian2024make}} & 79.72 & 67.15 & 83.47 & 74.81 & 72.92 & 62.35 & 49.59 & 29.03 & 19.72 (\textcolor{forestred}{26.9\%}) \\
& \textit{+C-Flat Turbo} & \textbf{79.82} & \textbf{68.54} & \textbf{84.00} & \textbf{75.12} & \textbf{73.11} & \textbf{62.38} & \textbf{50.49} & \textbf{29.30} & 45.89 (\textcolor{forestgreen}{62.6\%}) \\
\cmidrule(lr){2-11}
& MEMO~\cite{zhou2022model} & 82.26 & 73.89 & 86.66 & 79.90 & 70.96 & 61.05 & 56.22 & 38.32 & 135.14 (100\%) \\
& \textit{+C-Flat~\cite{bian2024make}} & 82.61 & 75.49 & 87.03 & 80.73 & 71.69 & 62.93 & 56.50 & 39.45 & 46.11 (\textcolor{forestred}{34.1\%}) \\
& \textit{+C-Flat Turbo} & \textbf{83.02} & \textbf{75.76} & \textbf{87.15} & \textbf{80.92} & \textbf{72.38} & \textbf{63.55} & \textbf{57.52} & \textbf{40.62} & 91.91 (\textcolor{forestgreen}{68.0\%}) \\
\midrule
\multirow{9}{*}{\centering \rotatebox{90}{\textit{PTM-based}}}  
& L2P~\cite{wang2022learning} & 89.36 & 83.94 & 73.04 & 59.14 & 78.05 & 72.60 & 64.18 & 52.10 & 110.29 (100\%) \\
& \textit{+C-Flat~\cite{bian2024make}} & 89.56 & 84.35 & \textbf{74.36} & \textbf{62.11} & 78.67 & 73.78 & 64.53 & 52.47 & 28.63 (\textcolor{forestred}{30.0\%}) \\
& \textit{+C-Flat Turbo} & \textbf{89.78} & \textbf{84.69} & 74.12 & 61.97 & \textbf{78.86} & \textbf{73.82} & \textbf{64.64} & \textbf{52.55} & 65.50 (\textcolor{forestgreen}{59.4\%}) \\
\cmidrule(lr){2-11}
& Ranpac~\cite{mcdonnell2024ranpac} & 94.32 & 90.72 & 92.61 & 88.68 & 82.07 & 76.80 & 71.66 & 60.17 & 154.64 (100\%) \\
& \textit{+C-Flat~\cite{bian2024make}} & 94.41 & 90.70 & 92.67 & 88.76 & 82.66 & 77.25 & 72.15 & 60.33 & 42.98 (\textcolor{forestred}{27.8\%}) \\
& \textit{+C-Flat Turbo} & \textbf{94.45} & \textbf{90.74} & \textbf{93.12} & \textbf{89.02} & \textbf{83.13} & \textbf{77.83} & \textbf{72.16} & \textbf{60.33} & 94.34 (\textcolor{forestgreen}{61.0\%}) \\
\cmidrule(lr){2-11}
& EASE~\cite{zhou2024expandable} & 91.91 & 87.30 & 89.16 & 83.96 & 80.49 & 75.05 & 64.38 & 52.02 & 166.67 (100\%) \\
& \textit{+C-Flat~\cite{bian2024make}} & 92.05 & 87.91 & 89.37 & 84.05 & 80.97 & 75.64 & 64.89 & 52.47 & 44.25 (\textcolor{forestred}{26.5\%}) \\
& \textit{+C-Flat Turbo} & \textbf{92.36} & \textbf{87.96} & \textbf{89.56} & \textbf{84.18} & \textbf{81.18} & \textbf{75.76} & \textbf{64.96} & \textbf{52.61} & 102.74 (\textcolor{forestgreen}{61.6\%}) \\
\bottomrule
\end{tabular}
}
\end{table*}

\subsection{Convergence Analysis}
We analyze the convergence of Turbo in the $k-1$ steps where the optimizer uses surrogate gradients instead of recomputing the full C-Flat gradients. Since C-Flat itself has been proven to converge in~\cite{bian2024make,zhang2023gradient}, we only need to control the additional approximation error introduced by these surrogate steps. A detailed convergence proof of Turbo under this surrogate-gradient view is provided in Appendix~\ref{sec:turbo-conv-appendix}.

\section{Experiments}
\label{sec:experiments}

\subsection{Experiment Setup}
\label{subsec:setup}
\textbf{Datasets.} Following~\cite{zhou2024continual}, we perform the evaluation on CIFAR100~\cite{krizhevsky2009learning}, CUB200~\cite{wah2011caltech}, ImageNet-R~\cite{hendrycks2021many} (IN-R), and ObjectNet~\cite{barbu2019objectnet} (ObjNet).  
These datasets contain 100 classes in CIFAR100, 200 classes in CUB200, and cover ImageNet-R and ObjectNet, which exhibit large domain gaps relative to the pre-trained datasets (ImageNet~\cite{deng2009imagenet}).
Following~\cite{zhou2024continual, zhou2023class}, we denote the data split as ‘B-$m$ Inc-$n$’, meaning that the initial task contains $m$ classes, and each subsequent task contains $n$ classes. The random seed for class-order shuffling is fixed at 1993~\cite{zhou2023pycil, zhou2024continual}. 

\noindent\textbf{Baselines}. Typical CL and pre-trained model (PTM)-based CL methods~\cite{zhou2024continual} are used to assess C-Flat Turbo. For the former, we cover the classical iCaRL~\cite{rebuffi2017icarl} and MEMO~\cite{zhou2022model} methods. 
For the latter, we compare against L2P~\cite{wang2022learning}, Ranpac~\cite{mcdonnell2024ranpac}, and EASE~\cite{zhou2024expandable}, spanning common CL categories.

\noindent\textbf{Implementation details.} All experiments and compared methods are implemented and reproduced in PyTorch and PILOT~\cite{zhou2024continual, zhou2023pycil, bian2024make} on a single RTX 3090 GPU. Unless otherwise specified, all hyperparameters and configurations remain unchanged from the open-source repository~\cite{zhou2024continual}. 
To ensure a fair comparison, we evaluate all methods with the same model and vanilla SGD optimizer (Adam for L2P) and adopt ViT-B/16-IN1K as the representative pre-trained models~\cite{wang2022learning, zhou2023revisiting}. 
Implementation details can be found in \cref{sec:param}. For evaluation, we primarily present the results in terms of average accuracy (Avg), final task accuracy (Last), and the average running speed (Img/s).

\subsection{Faster and Stronger Performance}
\label{subsec:perform}

We thoroughly evaluate the performance of C-Flat Turbo. As shown in Table~\ref{tab:ptm}, although pre-trained models exhibit strong generalization capabilities, their feature spaces remain susceptible to contamination during continual adaptation to evolving data distributions, thereby exacerbating catastrophic forgetting. While C-Flat mitigates this degradation through flat region search mechanisms, it incurs significant computational overhead.
On the one hand, C-Flat Turbo addresses this limitation by reusing flatness-related shortcut directions extracted from earlier steps, without repeated computation. On the other hand, it flexibly combines C-Flat with base optimizers through a stage-wise step schedule and an adaptive trigger for regularization, further accelerating the training process. A more detailed per-task accuracy progression and ablation studies are provided in the Appendix~\ref{sec:per-task}. Experimental results demonstrate that C-Flat Turbo achieves better accuracy than C-Flat, while significantly reducing training time. Notably, we find that C-Flat Turbo remains stable even in PTM scenarios with larger generalization gaps (e.g., CUB200, ImageNet-R, and ObjectNet). C-Flat Turbo trains CL models at about $2\times$ the speed of C-Flat, and $0.6\times$ that of SGD.
Overall, whether applied to typical CL benchmarks or to scenarios with large domain gaps in PTM, C-Flat Turbo maintains strong performance while offering superior training efficiency compared to the baseline C-Flat optimizer.

\begin{figure*}[t]
    \centering
    \begin{subfigure}{0.25\linewidth}
        \centering
        \includegraphics[width=1.\textwidth]{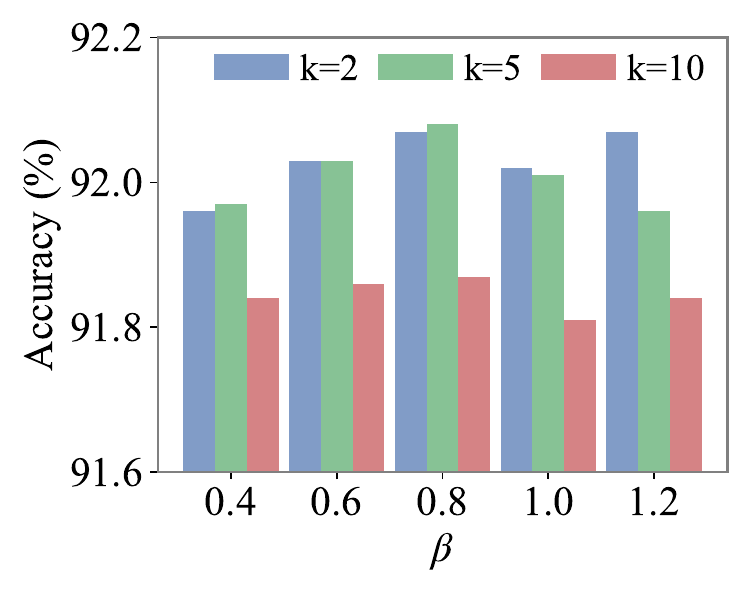}
        \caption{Sensitivity analysis}
        \label{subfig:params}
    \end{subfigure}\hfil
    \begin{subfigure}{0.24\linewidth}
        \centering
        \includegraphics[width=1.\textwidth]{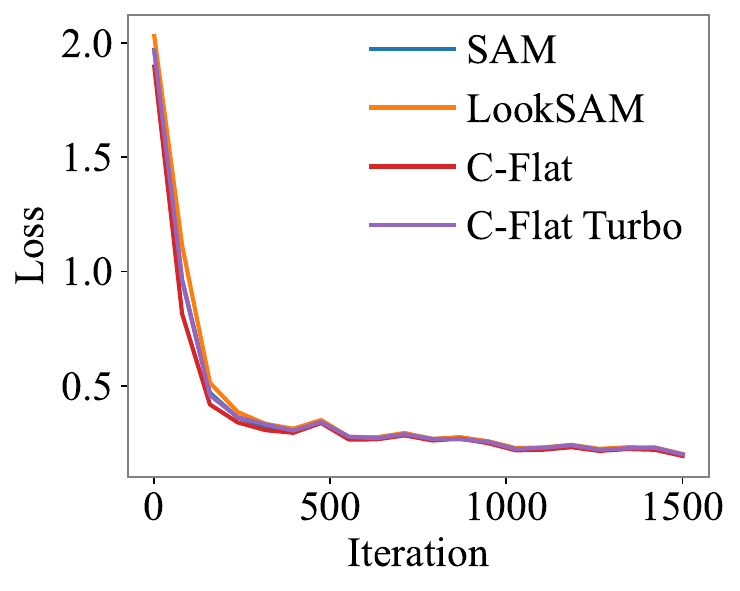}
        \caption{Variation on loss}
        \label{subfig:loss}
    \end{subfigure}\hfil
    \begin{subfigure}{0.245\linewidth}
        \centering
        \includegraphics[width=1.\textwidth]{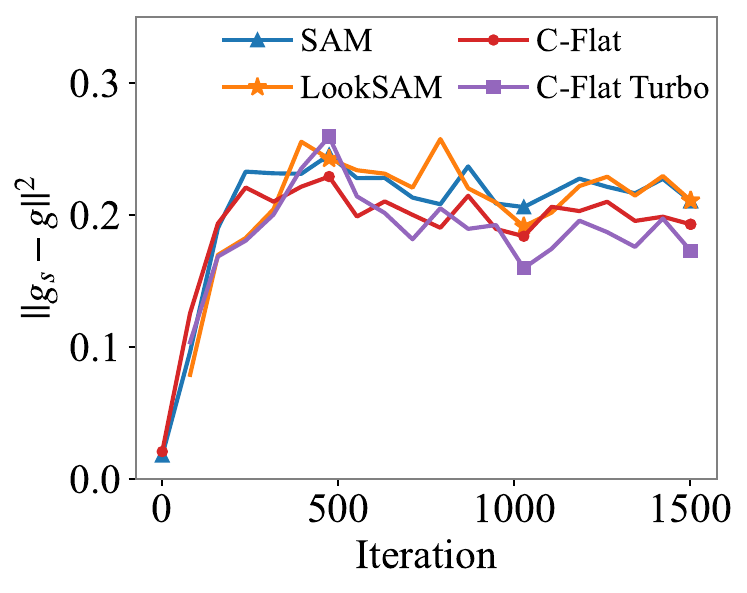}
        \caption{Variation on sharpness}
        \label{subfig:sharpness}
    \end{subfigure}\hfil
    \begin{subfigure}{0.245\linewidth}
        \centering
        \includegraphics[width=1.\textwidth]{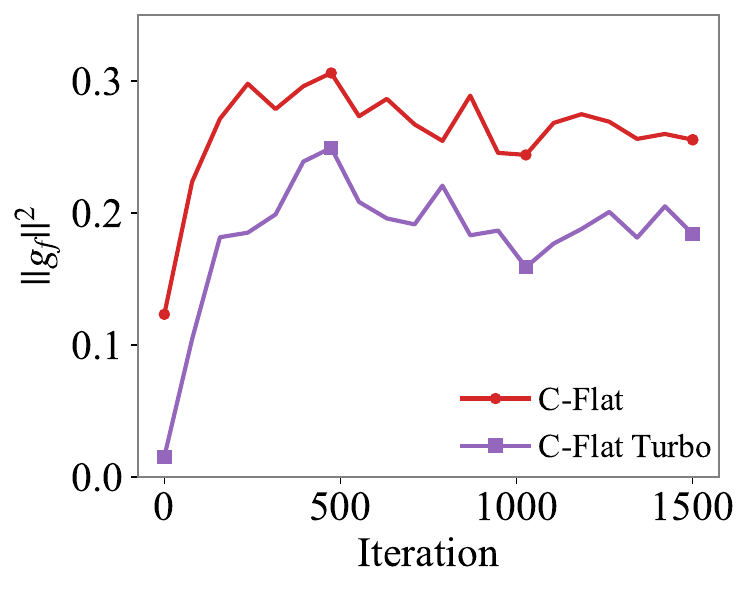}
        \caption{Variation on flatness}
        \label{subfig:flatness}
    \end{subfigure}
    \caption{(a) Sensitivity analysis of $k$. (b) - (d) Evolution of loss, sharpness and flatness on EASE w/ and w/o C-Flat Turbo.}
    \vspace{-3mm}
    \label{fig:ease}
\end{figure*}

\begin{table*}[t]
\vspace{-2mm}
\caption{Accuracy results with ResNet-18 and ResNet-34 trained from scratch. \textbf{Bolded} indicates the best result.}
\label{tab:resnet}
\centering
\small
\begin{tabular}{l cccccc} 
\toprule
\multirow{2}{*}{Method} & \multicolumn{3}{c}{ResNet-18} & \multicolumn{3}{c}{ResNet-34} \\ 
                        & Avg$\uparrow$ & Last$\uparrow$ & Img/s$\uparrow$
                        & Avg$\uparrow$ & Last$\uparrow$ & Img/s$\uparrow$            \\ 
\midrule
iCaRL~\cite{rebuffi2017icarl}          
                        & 59.13$_{\pm0.30}$  & 41.23$_{\pm0.91}$ & 2333.3 (100\%)
                        & 58.80$_{\pm1.04}$  & 41.26$_{\pm0.99}$ & 1250.4 (100\%) \\
\textit{+C-Flat~\cite{bian2024make}}  
                        & 59.45$_{\pm0.18}$  & 42.47$_{\pm0.06}$ & 686.3 (\textcolor{forestred}{29.4\%})
                        & 59.55$_{\pm0.93}$  & 42.09$_{\pm0.61}$ & 359.8 (\textcolor{forestred}{28.8\%}) \\
\textit{+C-Flat Turbo}      
                        & \bf{59.84$_{\mathbf{\pm0.05}}$}  & \bf{42.84$_{\mathbf{\pm0.18}}$} 
                        & 1750.1 (\textcolor{forestgreen}{75.0\%})
                        & \bf{59.75$_{\mathbf{\pm0.55}}$}  & \bf{42.34$_{\mathbf{\pm0.54}}$} 
                        & 960.6 (\textcolor{forestgreen}{76.8\%}) \\ 
\midrule
MEMO~\cite{zhou2022model}               
                        & 48.63$_{\pm0.78}$  & 29.19$_{\pm0.89}$ & 2413.8 (100\%)
                        & 68.49$_{\pm1.74}$  & 57.05$_{\pm1.46}$ & 1873.2 (100\%) \\
\textit{+C-Flat~\cite{bian2024make}}                        
                        & 49.98$_{\pm0.61}$  & 30.76$_{\pm0.57}$ & 886.1 (\textcolor{forestred}{36.7\%})
                        & 69.00$_{\pm1.39}$  & 59.29$_{\pm0.73}$ & 569.1 (\textcolor{forestred}{30.4\%}) \\
\textit{+C-Flat Turbo}      
                        & \bf{50.51$_{\mathbf{\pm0.55}}$}  & \bf{32.24$_{\mathbf{\pm0.27}}$} 
                        & 1891.9 (\textcolor{forestgreen}{78.4\%}) 
                        & \bf{69.48$_{\mathbf{\pm1.25}}$}  & \bf{59.33$_{\mathbf{\pm0.65}}$} 
                        & 1372.5 (\textcolor{forestgreen}{73.3\%}) \\
\bottomrule
\end{tabular}
\end{table*}

\subsection{Training From Scratch}
\label{subsec:scratch}

Table~\ref{tab:resnet} presents the accuracy results of iCaRL and MEMO with ResNet-18 and ResNet-34 trained from scratch. Notably, iCaRL benefits significantly from C-Flat Turbo, achieving an increase of 1.61\% in last accuracy on ResNet-18 and 1.08\% on ResNet-34. Similarly, MEMO exhibits substantial gains, particularly in final stage accuracy, where C-Flat Turbo improves performance by 3.05\% on ResNet-18 and 2.28\% on ResNet-34. Moreover, C-Flat Turbo forgets less than C-Flat across training, likely due to its softer sharpness constraint around local minima. The consistent improvements across different backbone architectures highlight the robustness of C-Flat Turbo in mitigating catastrophic forgetting while maintaining computational efficiency. Furthermore, the larger performance gains observed in MEMO suggest that C-Flat Turbo is particularly beneficial for expansion-based approaches, which involve multiple module updates and often struggle with stability and adaptability. These findings validate C-Flat Turbo as a highly effective strategy for continual learning, offering substantial accuracy improvements while enabling flexible training strategies.

\begin{table*}
\caption{Accuracy and training speeds of L2P and EASE across various optimizers. Task orders are shuffled for evaluation on dynamic sequences. \textcolor{forestred}{Red} and \textcolor{forestgreen}{green} denote the baseline and the efficient optimizer. \textbf{Bolded} indicates the best result.}
\label{tab:other_optim}
\centering
\small
\begin{tabular}{lcccccc}
\toprule
\multirow{2}{*}{Method} & \multicolumn{3}{c}{L2P~\cite{wang2022learning}}  & \multicolumn{3}{c}{EASE~\cite{zhou2024expandable}}               \\
                        & Avg$\uparrow$         & Last$\uparrow$        & Img/s$\uparrow$           & Avg$\uparrow$         & Last$\uparrow$        & Img/s$\uparrow$   \\
\midrule
Vanilla                 & 87.92$_{\pm1.30}$ & 83.66$_{\pm1.49}$ & 110.29 (100\%)  & 91.16$_{\pm0.71}$ & 87.49$_{\pm0.32}$ & 166.67 (100\%)  \\
\textit{+SAM~\cite{foret2020sharpness}}                    & 88.18$_{\pm1.39}$ & 83.84$_{\pm1.39}$ & 56.60 (\textcolor{forestred}{51.3\%}) & 91.36$_{\pm0.82}$ & 87.61$_{\pm0.27}$ & 86.71 (\textcolor{forestred}{52.0\%}) \\
\textit{+LookSAM~\cite{liu2022towards}}                & 88.35$_{\pm1.39}$ & 83.80$_{\pm1.20}$ & 88.76 (\textcolor{forestgreen}{80.5\%}) & 90.89$_{\pm0.86}$ & 86.99$_{\pm0.28}$ & 132.74 (\textcolor{forestgreen}{79.6\%}) \\
\textit{+C-Flat~\cite{bian2024make}}                 & 88.62$_{\pm0.88}$ & 84.04$_{\pm0.60}$ & 28.63 (\textcolor{forestred}{30.0\%}) & 91.60$_{\pm1.18}$ & 87.69$_{\pm0.50}$ & 44.25 (\textcolor{forestred}{26.5\%}) \\
\textit{+C-Flat Turbo}          & \bf{89.57$_{\mathbf{\pm0.36}}$} & \bf{84.48$_{\mathbf{\pm0.09}}$} 
& 65.50 \textcolor{forestgreen}{(59.4\%)} 
                                & \bf{91.75$_{\mathbf{\pm0.80}}$} & \bf{87.74$_{\mathbf{\pm0.20}}$} 
& 102.74 \textcolor{forestgreen}{(61.6\%)} \\
\bottomrule
\vspace{-4mm}
\end{tabular}
\end{table*}

\subsection{Comparison with Other Optimizers}
In Table~\ref{tab:other_optim}, we report the average accuracy, last accuracy, and training speeds on the CIFAR100 B0\_Inc10 setting compared to various zeroth-order optimizers. In particular, the update interval for the sharpness term in LookSAM and the flatness term in C-Flat Turbo is fixed to 5.

For the performance, as shown in Table~\ref{tab:other_optim}, we first observe that SAM and LookSAM do not offer obvious benefits over the vanilla optimizer, but C-Flat series shows significant improvement. The reason is that the backbone parameters loaded from the pre-trained model already possess strong generalization ability, resulting in uniformly low losses around local minima under various parameter perturbations. Nevertheless, C-Flat Turbo further refines the strong generalization of the pre-trained model by the horizontal and vertical components of the oracle gradient, thus boost CL.

For the training speeds, as concluded in Table~\ref{tab:other_optim}, although LookSAM significantly accelerates training compared to SAM by reusing historical gradients, it degrades performance on EASE due to its single zeroth-order regularization and the simplistic use of past iteration gradients.
C-Flat Turbo differs from LookSAM in that it progressively updates sharpness gradients to vanilla components and imposes more strict constraints to encourage convergence to a flatter region. Compared with C-Flat, our efficient method improves throughput from 30.0\% to 59.4\% on L2P and from 26.5\% to 61.6\% on EASE, making it even faster than SAM.

\subsection{Ablation Studies}
\textbf{Parameter Sensitivity.} In Figure~\ref{subfig:params}, we empirically investigate the sensitivity of the scale factor $\beta$ and the sampling step $k$ in EASE, trained under the CIFAR100 B0\_Inc10 setting. Here, $k$ denotes the frequency of leveraging the sharpness and flatness properties in C-Flat Turbo. For fairness, the sharpness and flatness steps are kept equal, meaning that C-Flat Turbo-$k$ indicates the cached sharpness and flatness gradients are updated every $k$ iterations. The hyperparameters $\rho$ and $\lambda$ are fixed at 0.05 and 0.2, respectively. \Cref{subfig:params} demonstrates that $\beta=0.8$ is the optimal choice in most cases. Moreover, when $k=2$ and $k=5$, the performance exhibits similar accuracy fluctuations as long as $\beta$ lies within the range [0.4, 1.0].

\noindent\textbf{Evolution of Sharpness and Flatness.}  
We approximate the sharpness and flatness gradients using cached branches and current SGD directions, following the same optimization procedure as C-Flat. This efficient approach does not slow convergence. As shown in Figure~\ref{subfig:loss}, C-Flat Turbo converges as fast as other optimizers. Figure~\ref{subfig:sharpness} illustrates that all optimizers start with low sharpness initially, owing to the pre-trained backbone’s generalization, and that the sharpness then declines alongside the training loss. Notably, the sharpness and flatness gradients in C-Flat Turbo converge to lower values than in C-Flat, due to the intermediate gradient descent steps being free of regularization constraints.

\subsection{Discussion on Scheduler Choices}
\label{subsec:sche}
  
Figure~\ref{fig:sche} compares the accuracy and training speed of MEMO and EASE. For the linear scheduler, we increase the step size $k$ with the task number $n$, following $k = 5 + 10 \cdot n / N$, whereas in the variant without a scheduler, $k$ is fixed at 5.
MEMO expands its architecture for each new task, which increases both the number of parameters and the computational cost, thereby prolonging training. The figure shows that C-Flat Turbo, with or without the scheduler consistently outperforms vanilla C-Flat in accuracy. In terms of speed, both variants are faster, and the scheduler yields an additional speedup of approximately 15\%.
In contrast, EASE reuses frozen adapters, which keeps the training cost stable. As $k$ increases, the scheduler provides roughly a 30\% faster over C-Flat while maintaining comparable accuracy.

\begin{figure}  
    \centering
    \begin{subfigure}{0.95\linewidth}
        \centering
        \includegraphics[width=\textwidth]{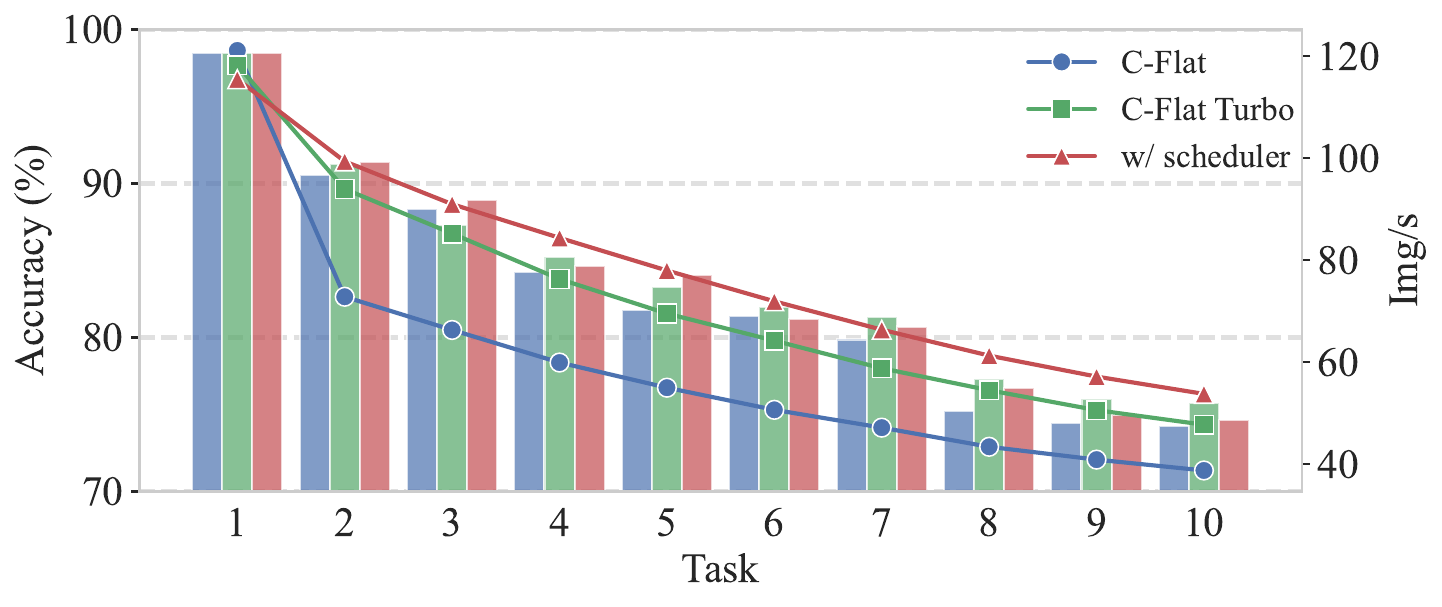}
    \end{subfigure}
    \hfil
    \begin{subfigure}{0.95\linewidth}
        \centering
        \includegraphics[width=\textwidth]{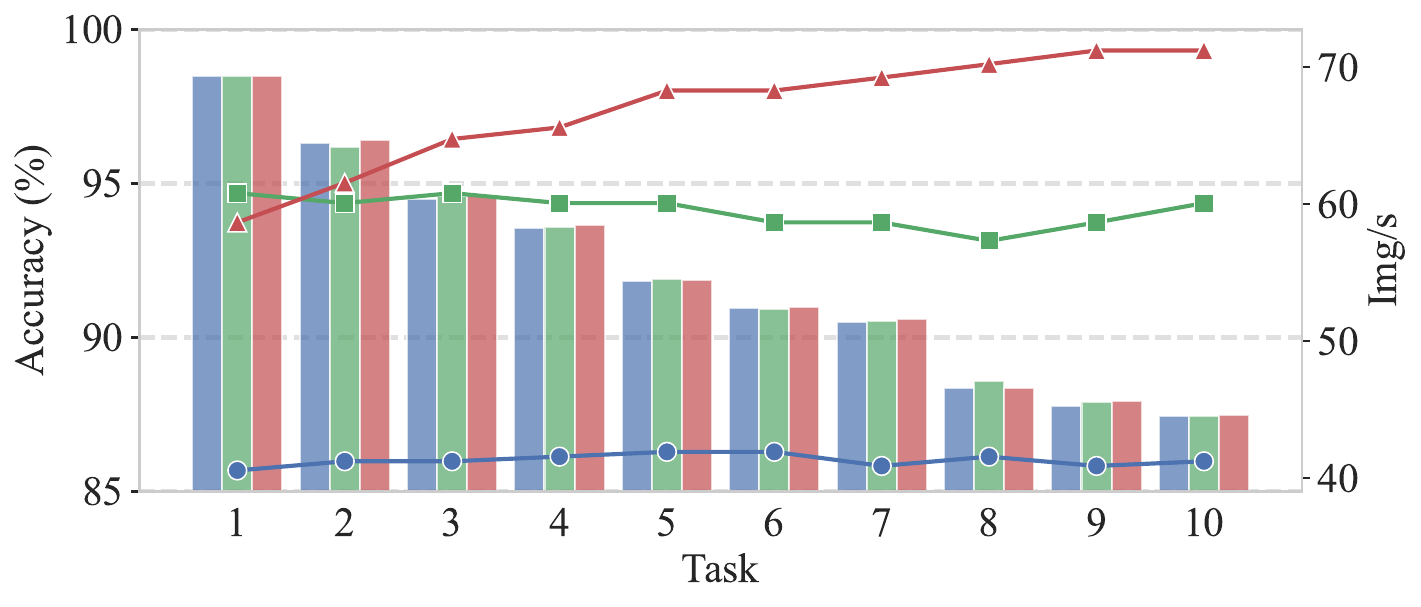}
    \end{subfigure}
    \vspace{-2mm}
    \caption{Accuracy and training speed comparison of different Turbo scheduling strategies.
      The upper part shows results on MEMO, while the lower part shows results on EASE.}
    \label{fig:sche}
    \vspace{-4mm}
  \end{figure}

\section{Conclusion}

This paper proposes C-Flat Turbo, an efficient optimizer that selectively takes shortcuts along stable directions toward flatness. By reusing historical sharpness and flatness signals instead of recomputing them at every step, C-Flat Turbo accelerates C-Flat while preserving its regularization effect. Building on the empirical observation that flatness gradients diminish across tasks, we further introduce (i) a linear scheduler that adaptively adjusts the turbo interval and (ii) an adaptive trigger that selectively activates C-Flat regularization only when it is most beneficial. Overall, our results reveal that adaptively modifying the oracle gradient can yield tangible efficiency gains in continual learning, and that C-Flat Turbo can be seamlessly plugged into a wide range of CL methods to provide consistent training speedups.

\textbf{Limitations.} 
Although C-Flat Turbo is more efficient than C-Flat, it remains costlier than standard optimizers due to the overhead of estimating the flatness component. Further reducing this cost, such as through cheaper finite-difference schemes or lower-rank approximations, is important future work. Moreover, our evaluation is limited to typical and PTM-based CIL benchmarks. Extending C-Flat Turbo to other CL scenarios, like vision-language models or reinforcement learning agents, may further reveal its generalization and robustness.

{
    \small
    \bibliographystyle{ieeenat_fullname}
    \bibliography{main}
}

\clearpage
\setcounter{page}{1}

\appendix

\section{Appendix}
\label{sec:appendix}
\subsection{Symbols and Notations of C-Flat Turbo}
\label{sec:symbols}
\begin{itemize}
\item model parameter: $\boldsymbol{\theta}$;
\item SAM perturbed model parameter: $\boldsymbol{\theta} + \boldsymbol{\epsilon}_0^* = \boldsymbol{\theta} + \rho \cdot \frac{\nabla \mathcal{L}(\boldsymbol{\theta})}{\|\nabla \mathcal{L}(\boldsymbol{\theta})\|}$;
\item proxy model parameter: $\boldsymbol{\theta} + \boldsymbol{\epsilon}_1^* = \boldsymbol{\theta} + \rho \cdot (\boldsymbol{g}_s - \boldsymbol{g}) / \| \boldsymbol{g}_s - \boldsymbol{g} \|$;
\item proxy perturbed model parameter: $\boldsymbol{\theta} + \boldsymbol{\epsilon}_1^* + \rho \cdot \nabla \mathcal{L}(\boldsymbol{\theta} + \boldsymbol{\epsilon}_1^*) / \| \nabla \mathcal{L}(\boldsymbol{\theta} + \boldsymbol{\epsilon}_1^*) \|$;
\item the empirical loss: $\mathcal{L}(\boldsymbol{\theta})$, with its gradient $\boldsymbol{g}$;
\item the SAM loss: $\mathcal{L}_{SAM}(\boldsymbol{\theta}) = \mathcal{L}(\boldsymbol{\theta}) + \mathcal{R}_{\rho}^0(\boldsymbol{\theta}) = \max_{\|\boldsymbol{\epsilon}_0\| \leq \rho} \; \mathcal{L}(\boldsymbol{\theta} + \boldsymbol{\epsilon}_0)$ with its gradient $\boldsymbol{g}_s$;
\item the C-Flat loss: $\mathcal{L}_{CFlat}(\boldsymbol{\theta}) = \mathcal{L}(\boldsymbol{\theta}) + \mathcal{R}_{\rho}^0(\boldsymbol{\theta}) + \lambda \cdot \mathcal{R}_{\rho}^1(\boldsymbol{\theta}) = \max_{\|\boldsymbol{\epsilon}_0\| \leq \rho} \; \mathcal{L}(\boldsymbol{\theta} + \boldsymbol{\epsilon}_0) + \lambda \cdot \rho  \max_{\|\boldsymbol{\epsilon}_1\| \leq \rho} \; 
\Big\| \nabla \mathcal{L}(\boldsymbol{\theta} + \boldsymbol{\epsilon}_1) \Big\|$, with its gradient $\boldsymbol{g}_s + \lambda\boldsymbol{g}_f$;
\item the gradient of proxy model: $\boldsymbol{g}_0 = \nabla \mathcal{L}(\boldsymbol{\theta} + \boldsymbol{\epsilon}_1^*)$;
\item the gradient of proxy perturbed model: $\boldsymbol{g}_1 = \nabla \mathcal{L}(\boldsymbol{\theta} + \boldsymbol{\epsilon}_1^* + \rho \cdot \nabla \mathcal{L}(\boldsymbol{\theta} + \boldsymbol{\epsilon}_1^*) / \| \nabla \mathcal{L}(\boldsymbol{\theta} + \boldsymbol{\epsilon}_1^*) \|)$;
\item the empirical loss term: $\boldsymbol{g} = \nabla \mathcal{L}(\boldsymbol{\theta})$;
\item the zeroth-order sharpness term: $\boldsymbol{g}_s - \boldsymbol{g} = \nabla \mathcal{R}_{\rho}^0(\boldsymbol{\theta})$;
\item the first-order flatness term: $\boldsymbol{g}_f = \nabla \mathcal{R}_{\rho}^1(\boldsymbol{\theta})$;
\item the direction-invariant sharpness component: $\boldsymbol{g}_{vs}
  = \boldsymbol{g}_s
   - \frac{\langle \boldsymbol{g}_s, \boldsymbol{g} \rangle}{\|\boldsymbol{g}\|^2}\,\boldsymbol{g}$;
\item the direction-invariant flatness component: $\boldsymbol{g}_{vf}
  = \boldsymbol{g}_f
   - \frac{\langle \boldsymbol{g}_f, \boldsymbol{g}_0 \rangle}{\|\boldsymbol{g}_0\|^2}\,\boldsymbol{g}_0$.
\end{itemize}

\subsection{Derivation of Equation~\ref{eq:grad_f}}
\label{subsec:derivation}

Following~\cite{bian2024make, zhang2023gradient}, the gradient of the first-order flatness loss $\mathcal{R}_{\rho}^1$ is:
\begin{align}
\label{eq:R1-grad-exact}
    \nabla_{\boldsymbol{\theta}} \mathcal{R}_{\rho}^1(\boldsymbol{\theta})   
    &= \rho \cdot \nabla_{\boldsymbol{\theta}} \max_{\boldsymbol{\epsilon} \in B(0, \rho)} \| \nabla \mathcal{L}(\boldsymbol{\theta} + \boldsymbol{\epsilon}) \| \\
    &= \rho \cdot \nabla_{\boldsymbol{\theta}} \| \nabla \mathcal{L}(\boldsymbol{\theta} + \boldsymbol{\epsilon}_1^*) \| \notag \\
    &= \rho \cdot \left( \frac{\partial}{\partial \boldsymbol{\theta}} \| \nabla \mathcal{L}(\boldsymbol{\theta} + \boldsymbol{\epsilon}_1^*) \| 
    + \frac{\partial \boldsymbol{\epsilon}_1^*}{\partial \boldsymbol{\theta}} \cdot \nabla_{\boldsymbol{\epsilon}} \| \nabla \mathcal{L}(\boldsymbol{\theta} + \boldsymbol{\epsilon}) \| \big|_{\boldsymbol{\epsilon} = \boldsymbol{\epsilon}_1^*} \right) \notag \\
    &\approx \rho \cdot \nabla_{\boldsymbol{\theta}} \| \nabla \mathcal{L}(\boldsymbol{\theta} + \boldsymbol{\epsilon}_1^*) \| \notag
\end{align}

Here, $\epsilon_1^*$ denotes the optimal perturbation that maximizes the gradient norm within the $\ell_2$-ball $B(0, \rho)$. To make the computation tractable, we approximate it using first-order Taylor expansion and finite differences:
\begin{align}
    \boldsymbol{\epsilon}_1^* 
    &= \arg\max_{\boldsymbol{\epsilon} \in B(0, \rho)} \left\| \nabla \mathcal{L}(\boldsymbol{\theta} + \boldsymbol{\epsilon}) \right\| \\ \nonumber
    &\approx \arg\max_{\boldsymbol{\epsilon} \in B(0, \rho)} 
    \left( \left( \nabla_{\boldsymbol{\theta}} \left\| \nabla \mathcal{L}(\boldsymbol{\theta}) \right\| \right)^T \boldsymbol{\epsilon} \right)  \\ \nonumber
    &= \rho \cdot \frac{ \nabla_{\boldsymbol{\theta}} \left\| \nabla \mathcal{L}(\boldsymbol{\theta}) \right\| }
                { \left\| \nabla_{\boldsymbol{\theta}} \left\| \nabla \mathcal{L}(\boldsymbol{\theta}) \right\| \right\| } \\ \nonumber
    &\approx \rho \cdot \frac{ \nabla \mathcal{L}(\boldsymbol{\theta} + \boldsymbol{\delta}) - \nabla \mathcal{L}(\boldsymbol{\theta}) }
                        { \left\| \nabla \mathcal{L}(\boldsymbol{\theta} + \boldsymbol{\delta}) - \nabla \mathcal{L}(\boldsymbol{\theta}) \right\| }  \\ \nonumber
    &= \rho \cdot \frac{ \boldsymbol{g}_s - \boldsymbol{g} }{ \| \boldsymbol{g}_s - \boldsymbol{g} \| },
\end{align}
where $\boldsymbol{g} = \nabla \mathcal{L}(\boldsymbol{\theta})$, $\boldsymbol{g}_s = \nabla \mathcal{L}(\boldsymbol{\theta} + \boldsymbol{\delta})$, and $\boldsymbol{\delta} = \rho' \cdot \frac{\boldsymbol{g}}{\|\boldsymbol{g}\|}$ is a small perturbation in the direction of the gradient, with $\rho' \ll \rho$.

Let $\boldsymbol{\theta}_p = \boldsymbol{\theta} + \boldsymbol{\epsilon}_1^*$ be the perturbed model after maximizing the gradient norm. Then Eq.~\ref{eq:R1-grad-exact} continues as:
\begin{align}
\label{eq:g-f-fd}
    \nabla_{\boldsymbol{\theta}} \mathcal{R}_\rho^1(\boldsymbol{\theta})  
    &\approx \rho \cdot \nabla_{\boldsymbol{\theta}} \Big\| \nabla \mathcal{L}(\boldsymbol{\theta}_p) \Big\| \\ \nonumber
    &\approx \frac{\rho}{\rho'} \cdot \Big[ \nabla \mathcal{L}\Big(\boldsymbol{\theta}_p + \rho' \cdot \frac{ \nabla \mathcal{L}(\boldsymbol{\theta}_p) }{ \| \nabla \mathcal{L}(\boldsymbol{\theta}_p) \| } \Big) - \nabla \mathcal{L}(\boldsymbol{\theta}_p) \Big],
\end{align}
where $\rho' \ll \rho$ is a small step size for finite-difference approximation.

In our symbol table, we define
\begin{align}
    \boldsymbol{g}_0
    &:= \nabla \mathcal{L}(\boldsymbol{\theta}_p)
     = \nabla \mathcal{L}(\boldsymbol{\theta} + \boldsymbol{\epsilon}_1^*),
     \label{eq:g0-def}\\
    \boldsymbol{g}_1
    &:= \nabla \mathcal{L}\Big(
       \boldsymbol{\theta}_p
       + \rho' \cdot \frac{\nabla \mathcal{L}(\boldsymbol{\theta}_p)}
       {\|\nabla \mathcal{L}(\boldsymbol{\theta}_p)\|}
       \Big),
     \label{eq:g1-def}
\end{align}

Then Eq.~\ref{eq:g-f-fd} can be written as the finite-difference form
\begin{equation}
    \boldsymbol{g}_f
    \;\approx\; \frac{\rho}{\rho'} \bigl(\boldsymbol{g}_1 - \boldsymbol{g}_0\bigr).
    \label{eq:g-f-from-g0-g1}
\end{equation}
For simplicity of notation, we absorb the constant factor $\rho/\rho'$ into
$\lambda$, so that at the level of directions we can view
\[
  \boldsymbol{g}_f \propto \boldsymbol{g}_1 - \boldsymbol{g}_0.
\]

Finally, we define the direction-invariant component of $\boldsymbol{g}_f$
with respect to $\boldsymbol{g}_0$:
\begin{equation}
    \boldsymbol{g}_{vf}
    = \boldsymbol{g}_f - \frac{\langle \boldsymbol{g}_f, \boldsymbol{g}_0 \rangle}
      {\|\boldsymbol{g}_0\|^2} \cdot \boldsymbol{g}_0,
    \label{eq:gvf-def}
\end{equation}
which is orthogonal to $\boldsymbol{g}_0$ and is the quantity tracked
by the EMA variable in the implementation.

\subsection{Hyperparameter Settings}
\label{sec:param}

We report the main hyperparameter settings for methods trained on CIFAR100 in Table~\ref{app:hyper}. For the other datasets, we follow the original settings from the open-source repository and keep $k=5$ and $\beta=0.8$ fixed to balance efficiency and performance.

We also study the sensitivity of the scheduler and trigger threshold in C-Flat Turbo on CIFAR100. For the scheduler $k_t = k_0 + c \cdot \frac{n}{N}$, larger $k_0$ generally improves throughput with negligible performance changes, while $c=10$ provides a good trade-off between speed and accuracy, as shown in Table~\ref{app:sche}. This suggests that moderately increasing the turbo interval over the task sequence is sufficient to obtain most of the efficiency gains. For the trigger threshold $\|\boldsymbol{g}_{0j}\|^2 > \mu_{f,j} + m \cdot \sigma_{f,j}$, Table~\ref{app:trig} shows that $m$ between 0.2 and 1 performs best overall. Larger values improve efficiency but tend to degrade performance, whereas smaller values may lead to less stable trigger behavior.

\begin{table}[t]
\scriptsize
\centering
\caption{Hyperparameter settings for CIFAR100.}
\label{app:hyper}
\setlength{\tabcolsep}{1.2mm}
\renewcommand{\arraystretch}{1.2}
\begin{tabular}{cccccccccc}
\toprule
Methods & Epochs & LR     & BS & Tasks & Exemplar & $\rho$ & $\lambda$ & $k$ & $\beta$ \\
\midrule
iCaRL  & 20     & 1e-3   & 32 & 10    & 20       & 0.1    & 0.2       & 5   & 0.8      \\
MEMO   & 20     & 1e-3   & 32 & 10    & 20       & 0.1    & 0.2       & 5   & 0.8      \\
L2P    & 5      & 2e-3   & 16 & 10    & -        & 0.02   & 0.2       & 5   & 0.8      \\
Ranpac & 5      & 1e-2   & 16 & 10    & -        & 0.05   & 0.2       & 5   & 0.8      \\
EASE   & 5      & 2.5e-3 & 16 & 10    & -        & 0.05   & 0.2       & 5   & 0.8      \\
\bottomrule
\end{tabular}
\end{table}

\begin{table}[h]
\scriptsize
\centering
\caption{Sensitivity analysis of the scheduler hyperparameters in C-Flat Turbo on CIFAR100.}
\label{app:sche}
\setlength{\tabcolsep}{1.2mm}
\renewcommand{\arraystretch}{1.2}
\begin{tabular}{cccccccccc}
\toprule
Optimizer & \multicolumn{9}{c}{C-Flat Turbo} \\
\midrule
$k_0$  & 2     & 2     & 2     & 5     & 5     & 5     & 10    & 10    & 10    \\
$c$    & 5     & 10    & 20    & 5     & 10    & 20    & 5     & 10    & 20    \\
Avg    & 92.19 & 92.18 & 92.05 & 92.08 & 92.08 & 92.01 & 92.07 & 92.04 & 92.04 \\
Last   & 87.67 & 87.59 & 87.45 & 87.62 & 87.54 & 87.42 & 87.51 & 87.59 & 87.46 \\
Img/s  & 99.17 & 99.91 & 63.19 & 104.86 & 102.74 & 99.05 & 104.64 & 106.89 & 107.39 \\
\bottomrule
\end{tabular}
\end{table}

\begin{table}[h]
\scriptsize
\centering
\caption{Sensitivity analysis of the trigger threshold hyperparameter in C-Flat Turbo on CIFAR100.}
\label{app:trig}
\setlength{\tabcolsep}{1.2mm}
\renewcommand{\arraystretch}{1.2}
\begin{tabular}{ccccccc}
\toprule
$m$   & 0.1   & 0.2   & 0.5   & 1     & 2     & 5     \\
\midrule
Avg   & 91.98 & 92.07 & 92.05 & 92.06 & 91.86 & 91.94 \\
Last  & 87.43 & 87.48 & 87.39 & 87.50 & 87.07 & 87.16 \\
Img/s & 75.56 & 105.43 & 97.69 & 102.74 & 154.52 & 135.18 \\
\bottomrule
\end{tabular}
\end{table}

\subsection{Memory Usage}
The cached gradients are used to substitute partial sharpness-aware gradient computations, so their memory usage heavily depends on the number of trainable parameters in the model. As shown in Table~\ref{tab:memory}, although larger models require more cached gradients, the overall memory overhead remains almost negligible relative to the expansion typically caused by large architectures.

\begin{table}[h]
\centering
\scriptsize
\caption{Memory usage of different architectures. EASE uses a frozen backbone, while iCaRL updates the full backbone.}
\label{tab:memory}
\renewcommand{\arraystretch}{1.2}
\setlength{\tabcolsep}{1.2mm}
\begin{tabular}{l cccc}
\toprule
Method & Optimizer & Backbone & Trainable / total params & Memory \\
\midrule
\multirow{4}{*}{EASE}
& C-Flat & ViT-Base-16  & 1.19M / 86.99M   & 2.14GB \\
& C-Flat Turbo  & ViT-Base-16  & 1.19M / 86.99M   & 2.15GB \\
& C-Flat & ViT-Large-16 & 3.17M / 306.47M  & 5.34GB \\
& C-Flat Turbo  & ViT-Large-16 & 3.17M / 306.47M  & 5.37GB \\
\midrule
\multirow{4}{*}{iCaRL}
& C-Flat & ResNet-18    & 11.17M / 11.17M  & 1.55GB \\
& C-Flat Turbo  & ResNet-18    & 11.17M / 11.17M  & 1.66GB \\
& C-Flat & ResNet-34    & 21.28M / 21.28M  & 2.32GB \\
& C-Flat Turbo  & ResNet-34    & 21.28M / 21.28M  & 2.51GB \\
\bottomrule
\end{tabular}
\end{table}

\subsection{Comparison to Similar Works}
We further compare C-Flat Turbo with SS-SAM and AE-SAM on L2P and EASE, as shown in Table~\ref{app:sams}. Overall, C-Flat Turbo consistently achieves the best performance on both benchmarks, while maintaining competitive efficiency.

\begin{table}[h]
\scriptsize
\centering
\caption{Comparison with SS-SAM and AE-SAM on L2P (left) and EASE (right).}
\label{app:sams}
\setlength{\tabcolsep}{1.2mm}
\renewcommand{\arraystretch}{1.2}
\begin{tabular}{lccc ccc}
\toprule
& \multicolumn{3}{c}{L2P} & \multicolumn{3}{c}{EASE} \\
Method & Avg & Last & Img/s & Avg & Last & Img/s \\
\midrule
+SS-SAM       & 89.50 & 84.45 & 77.42 (70.20\%) & 90.90 & \underline{87.87} & 91.65 (54.99\%) \\
+AE-SAM       & \underline{89.70} & \underline{84.48} & 80.86 (73.32\%) & \underline{92.03} & 87.37 & 112.58 (67.54\%) \\
+C-Flat Turbo & \textbf{89.78} & \textbf{84.69} & 65.50 (59.4\%) & \textbf{92.36} & \textbf{87.96} & 102.74 (61.6\%) \\
\bottomrule
\end{tabular}
\end{table}

\subsection{Other CL Settings}
We further evaluate C-Flat Turbo with DUCT on DIL and with DualPrompt on TIL and CIL. As shown in Tables~\ref{app:dil} and~\ref{app:til}, Turbo consistently improves performance over the corresponding baselines.

\begin{table}[h]
\scriptsize
\centering
\caption{Results on DomainNet under the DIL setting.}
\label{app:dil}
\setlength{\tabcolsep}{1.2mm}
\renewcommand{\arraystretch}{1.2}
\begin{tabular}{lccccccc}
\toprule
Method & Clipart & Infograph & Painting & Quickdraw & Real & Sketch & Avg \\
\midrule
DUCT
& 74.21 & 60.52 & 67.73 & \underline{63.50} & 69.51 & 68.63 & 67.35 \\
+C-Flat
& \textbf{74.93} & \underline{60.88} & \textbf{68.10} & 63.30 & \underline{70.22} & \underline{69.41} & \underline{67.81} \\
+C-Flat Turbo
& \underline{74.89} & \textbf{60.91} & \underline{68.01} & \textbf{63.75} & \textbf{70.82} & \textbf{69.50} & \textbf{67.98} \\
\bottomrule
\end{tabular}
\end{table}

\begin{table}[h]
\scriptsize
\centering
\caption{Results on CIFAR100 under the CIL and TIL settings.}
\label{app:til}
\setlength{\tabcolsep}{1.2mm}
\renewcommand{\arraystretch}{1.2}
\begin{tabular}{lcccc}
\toprule
\multirow{2}{*}{Method} & \multicolumn{2}{c}{CIL} & \multicolumn{2}{c}{TIL} \\
& Avg & Last & Avg & Last \\
\midrule
DualPrompt & 85.77 & \textbf{89.10} & 97.79 & \textbf{99.40} \\
+C-Flat     & 85.85 & 87.80           & 97.67 & 99.30           \\
+C-Flat Turbo      & \textbf{86.01} & 87.00   & \textbf{97.94} & \textbf{99.40} \\
\bottomrule
\end{tabular}
\end{table}

\subsection{Per-task Accuracy and Ablation Studies}
\label{sec:per-task}
Per-task accuracy provides a more detailed view of the continual learning process. As shown in Table~\ref{tab:pertask}, the reuse mechanism significantly reduces training time with minimal performance loss, while the linear step-size scheduler further improves speed, particularly for longer task sequences. The adaptive trigger additionally accelerates training, as it allows basic single propagation gradient descent in certain stages. 

Regarding performance gains, prior works have shown that selectively applying SAM updates can outperform applying SAM throughout training. For instance, SS-SAM~\cite{zhao2022stochastic} explicitly demonstrates that with appropriate scheduling, models can achieve comparable or even superior performance at substantially lower computational cost compared to training exclusively with SAM. Similar observations also have been reported for AE-SAM~\cite{jiang2023adaptive} and SAM-In-Late-Phase~\cite{zhou2024sharpness}. 

\begin{table*}[t]
\centering
\small
\caption{Per-task accuracy and ablation study results for EASE trained on the 10-split CIFAR100 dataset.}
\label{tab:pertask}
\setlength{\tabcolsep}{1.2mm}
\renewcommand{\arraystretch}{1.2}
\begin{tabular}{l c c c cccccccccc c c}
\toprule
Method & reuse & scheduler & trigger & T1 & T2 & T3 & T4 & T5 & T6 & T7 & T8 & T9 & T10 & Avg & Img/s \\
\midrule
EASE
& $\times$ & $\times$ & $\times$
& 98.40 & 96.25 & 94.63 & 93.88 & 91.80 & 90.92 & 90.47 & 88.09 & 87.58 & 87.17 & 91.92 & 166.67 \\

+C-Flat
& $\times$ & $\times$ & $\times$
& 98.50 & 96.45 & 94.87 & 94.08 & 91.94 & 91.05 & 90.64 & 88.44 & 87.93 & 87.58 & 92.15 & 44.25 \\

\multirow{3}{*}{+C-Flat Turbo}
& $\checkmark$ & $\times$ & $\times$
& 98.50 & 96.37 & 94.77 & 93.94 & 91.92 & 91.05 & 90.67 & 88.28 & 87.81 & 87.45 & 92.08 & 67.20 \\
& $\checkmark$ & $\checkmark$ & $\times$
& 98.40 & 96.31 & 94.74 & 93.89 & 91.90 & 91.00 & 90.70 & 88.25 & 87.75 & 87.40 & 92.03 & 74.63 \\
& $\checkmark$ & $\checkmark$ & $\checkmark$
& 98.50 & \textbf{96.60} & \textbf{95.07} & \textbf{94.15} & \textbf{92.08} & \textbf{91.27} & \textbf{90.73} & \textbf{88.52} & \textbf{88.00} & \textbf{87.57} & \textbf{92.25} & \textbf{102.74} \\
\bottomrule
\end{tabular}
\end{table*}

\subsection{Detailed Evolution of Gradient Distances}
\label{suppsec:diff_g_detail}

\begin{figure*}[t]
    \centering
    \begin{subfigure}[b]{0.19\textwidth}
        \includegraphics[width=\textwidth]{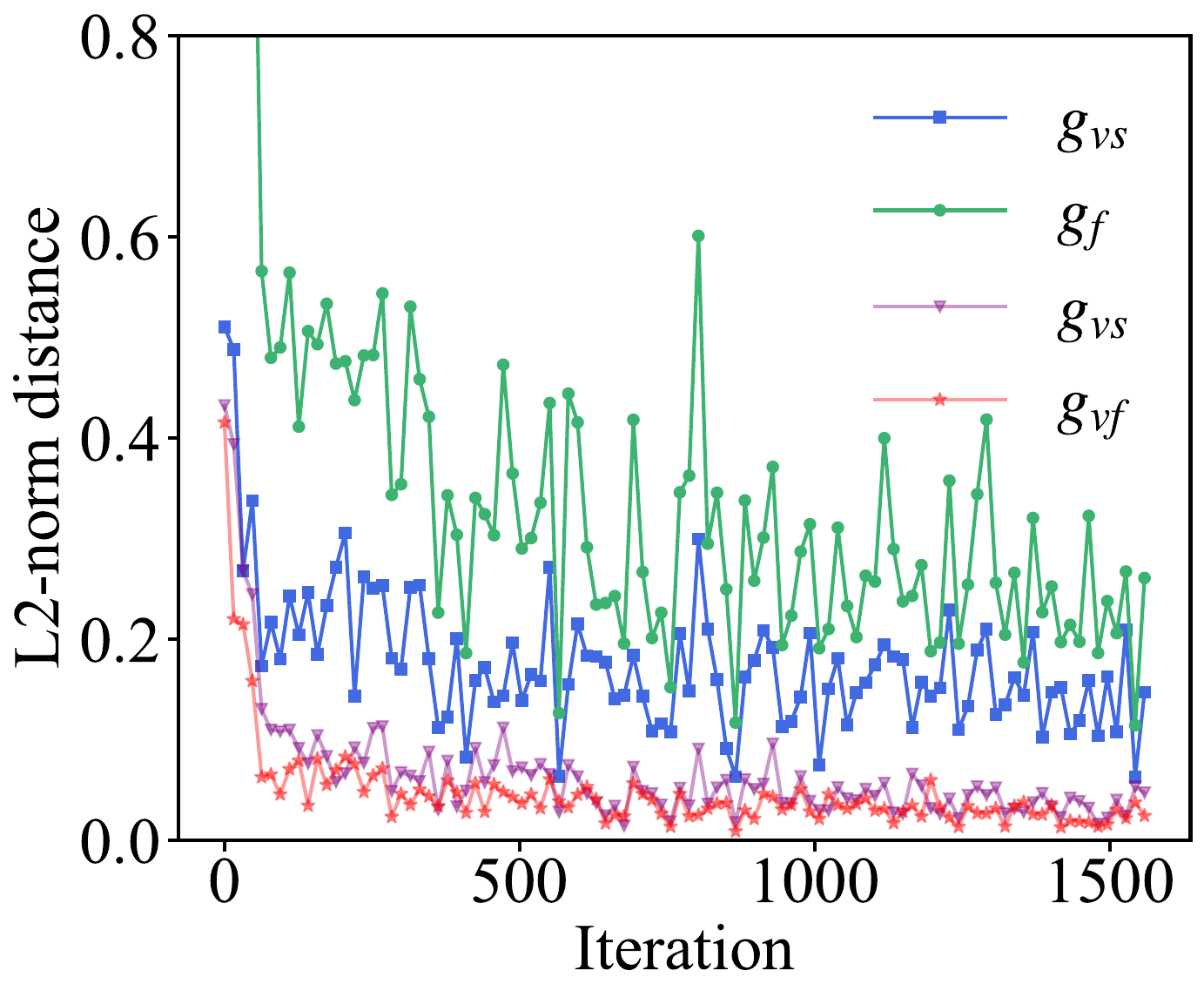}
        \caption{Task 0}
    \end{subfigure}
    \begin{subfigure}[b]{0.19\textwidth}
        \includegraphics[width=\textwidth]{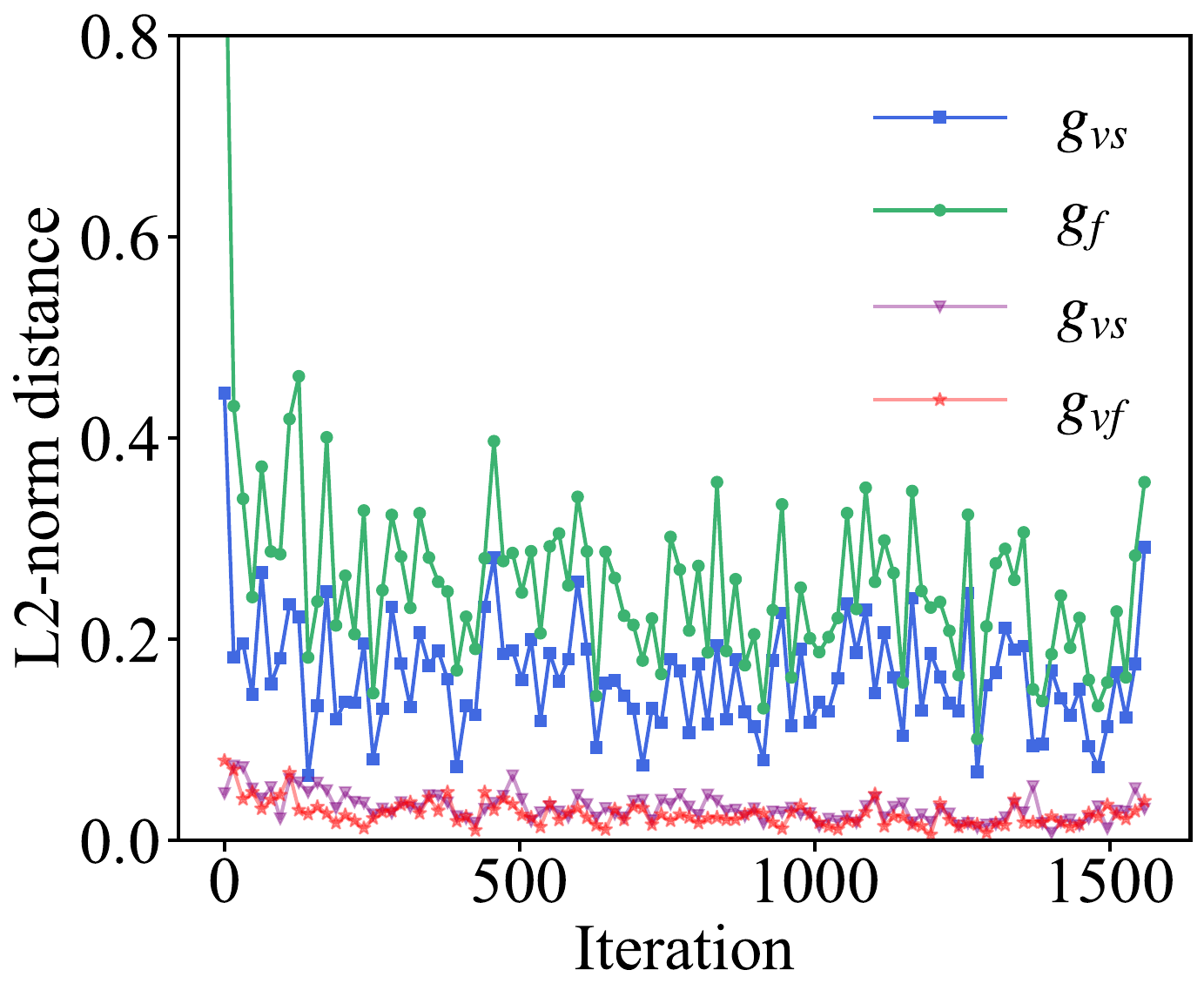}
        \caption{Task 1}
    \end{subfigure}
    \begin{subfigure}[b]{0.19\textwidth}
        \includegraphics[width=\textwidth]{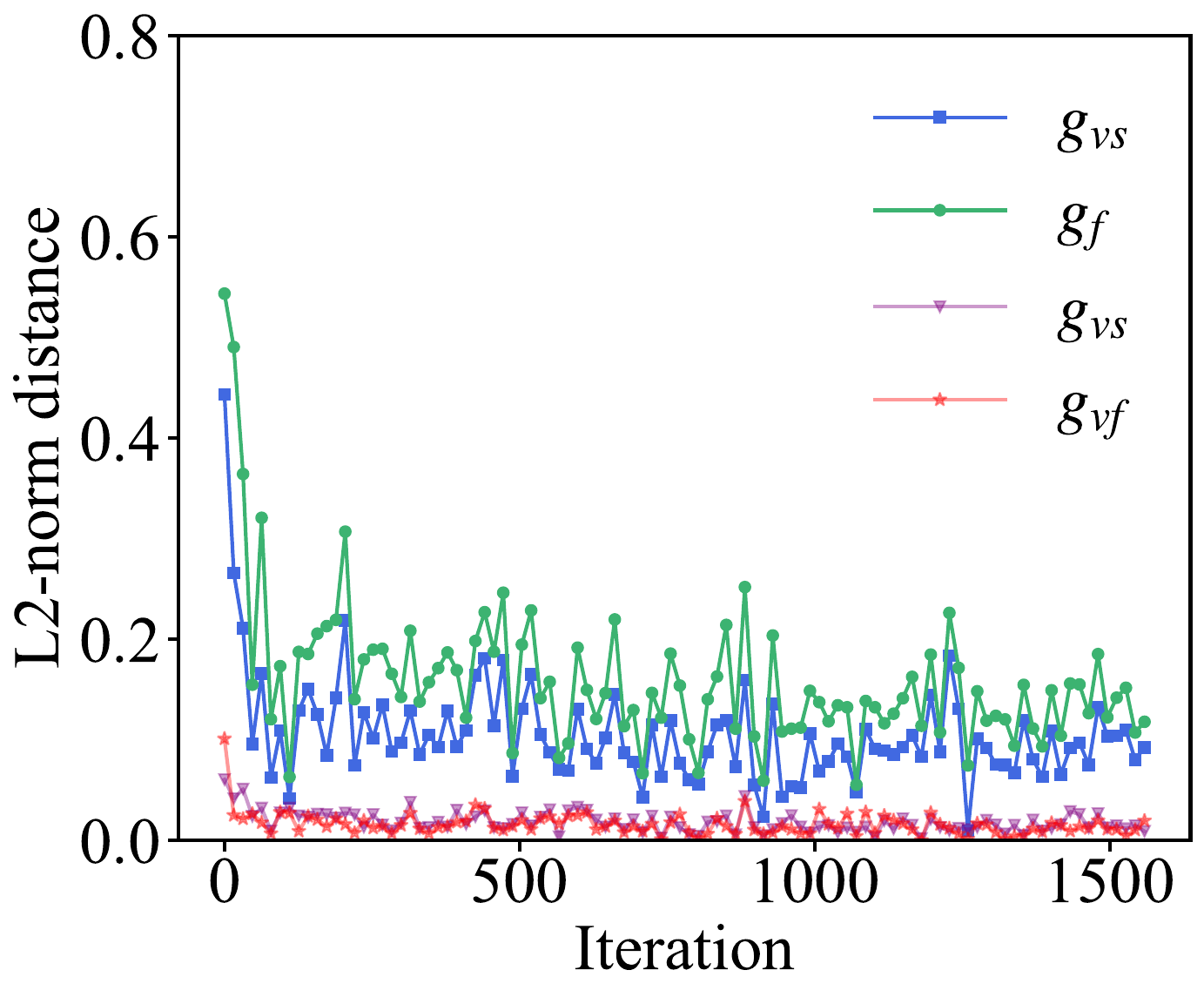}
        \caption{Task 2}
    \end{subfigure}
    \begin{subfigure}[b]{0.19\textwidth}
        \includegraphics[width=\textwidth]{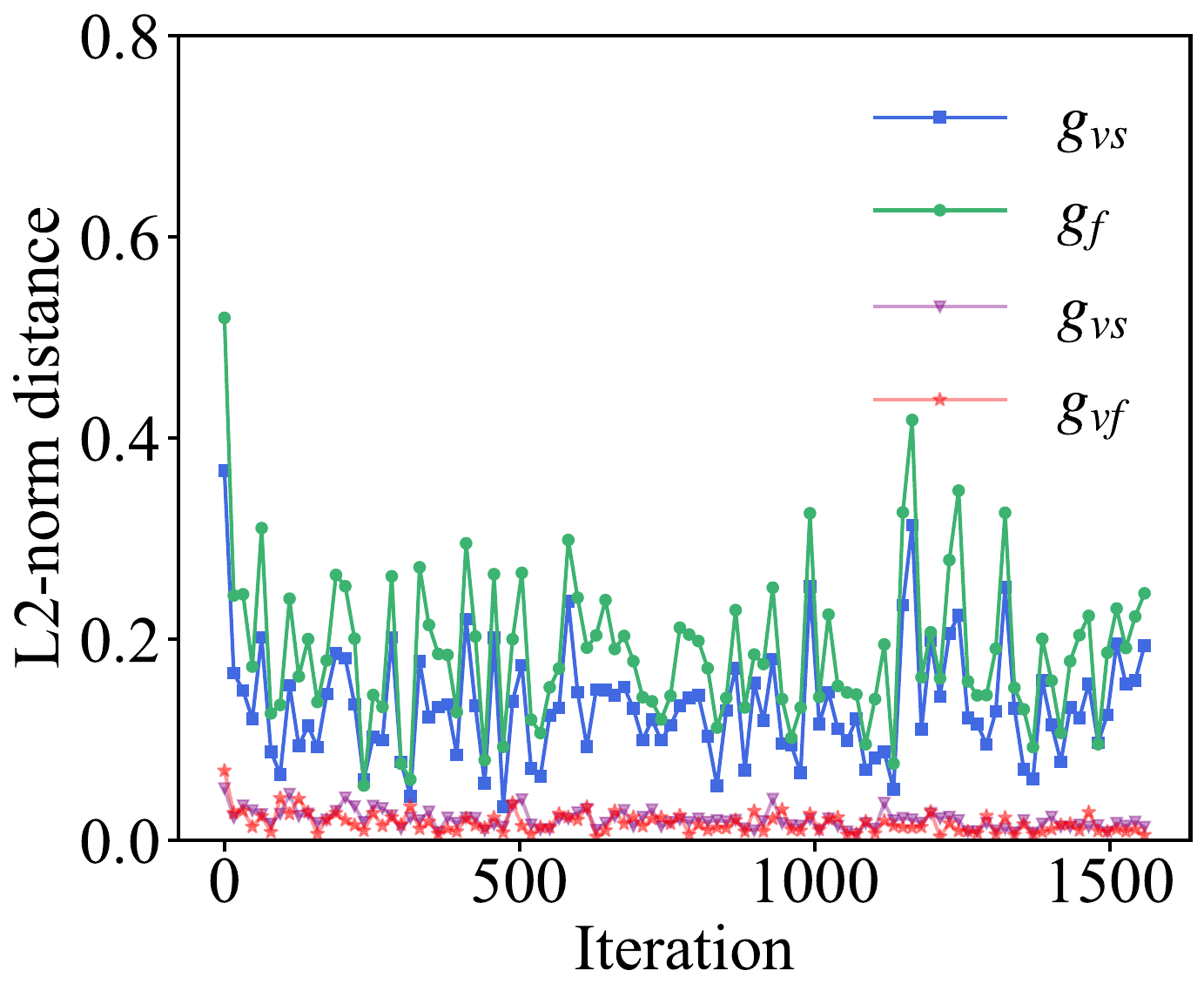}
        \caption{Task 3}
    \end{subfigure}
    \begin{subfigure}[b]{0.19\textwidth}
        \includegraphics[width=\textwidth]{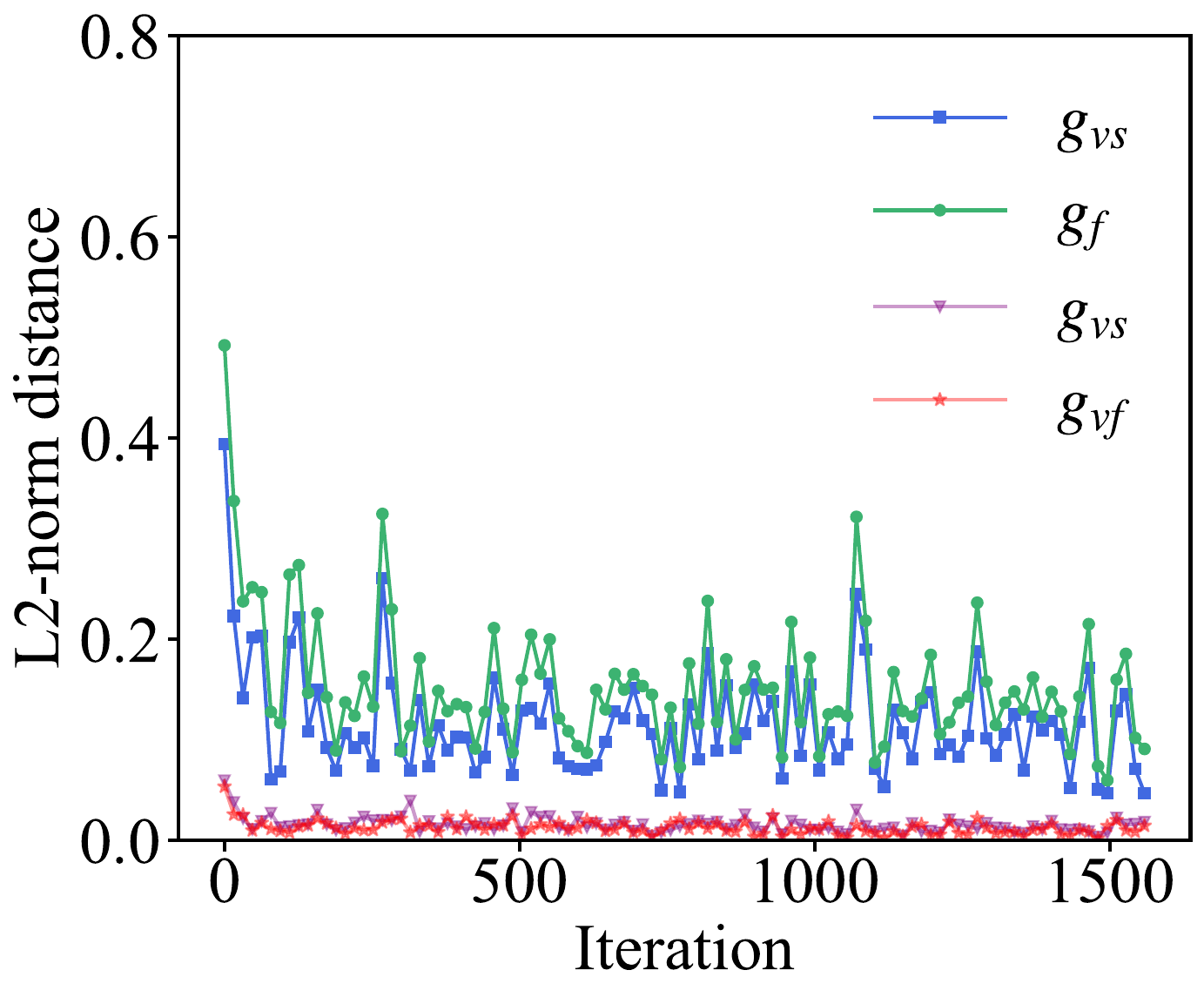}
        \caption{Task 4}
    \end{subfigure}
    \\
    \begin{subfigure}[b]{0.19\textwidth}
        \includegraphics[width=\textwidth]{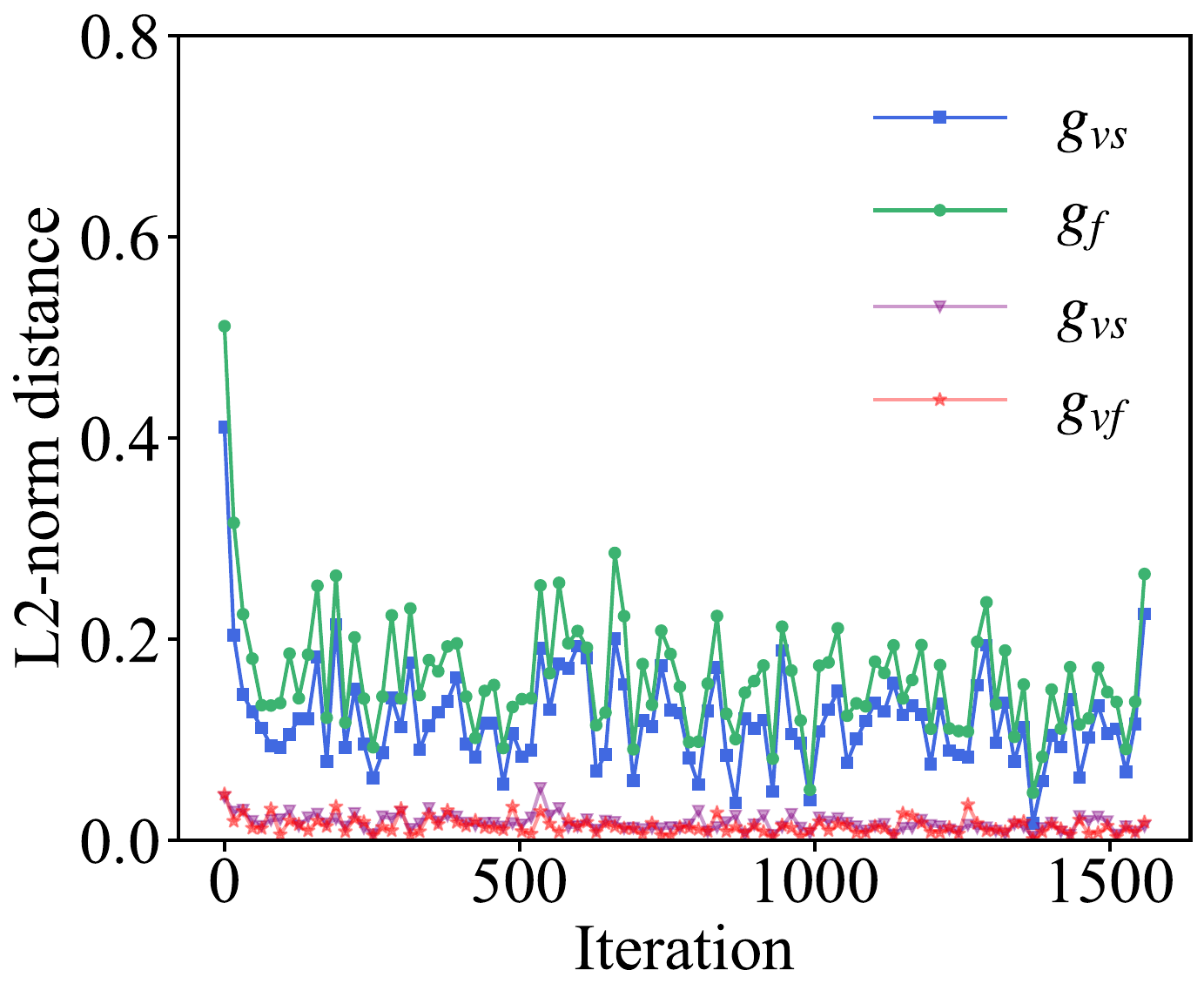}
        \caption{Task 5}
    \end{subfigure}
    \begin{subfigure}[b]{0.19\textwidth}
        \includegraphics[width=\textwidth]{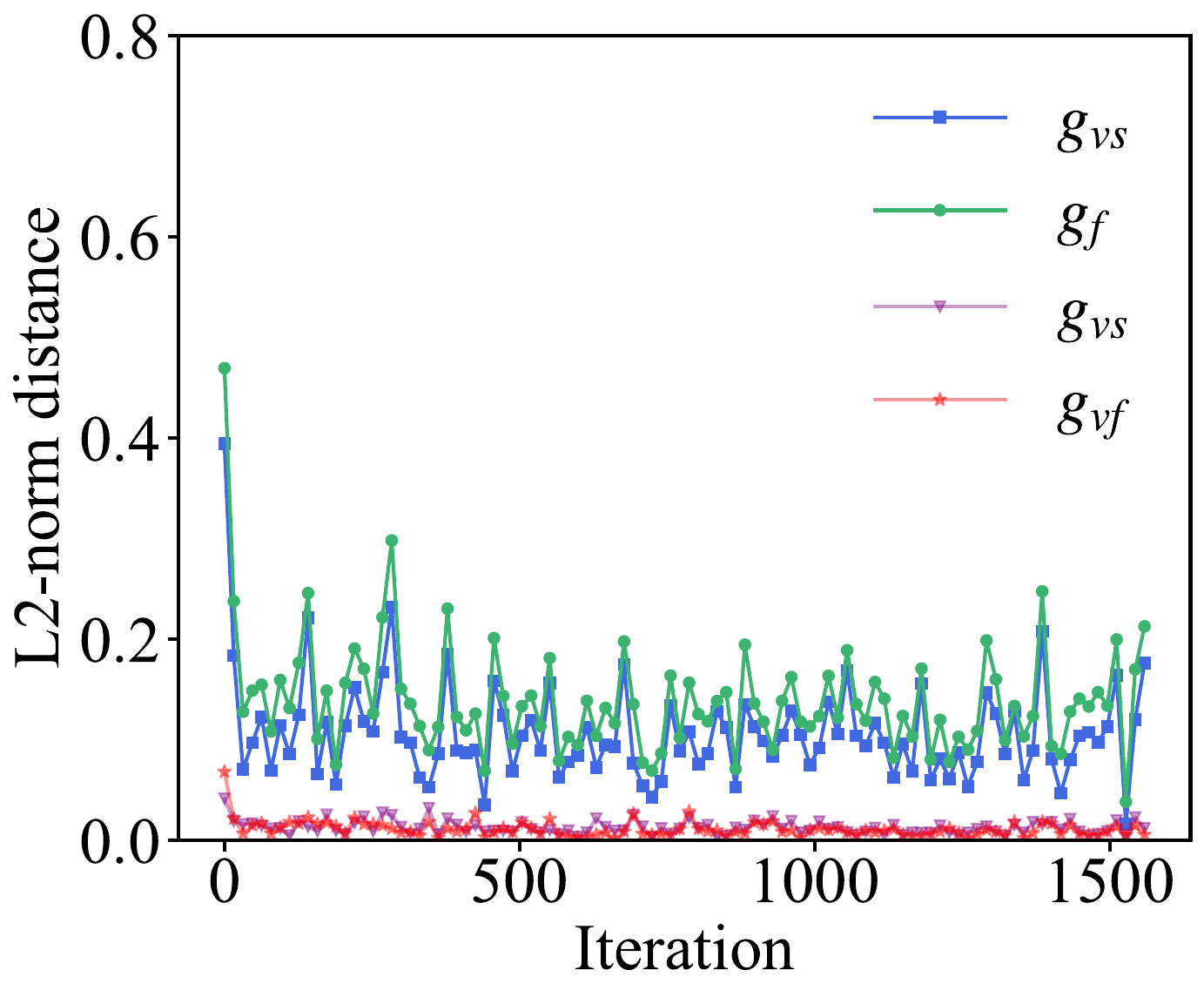}
        \caption{Task 6}
    \end{subfigure}
    \begin{subfigure}[b]{0.19\textwidth}
        \includegraphics[width=\textwidth]{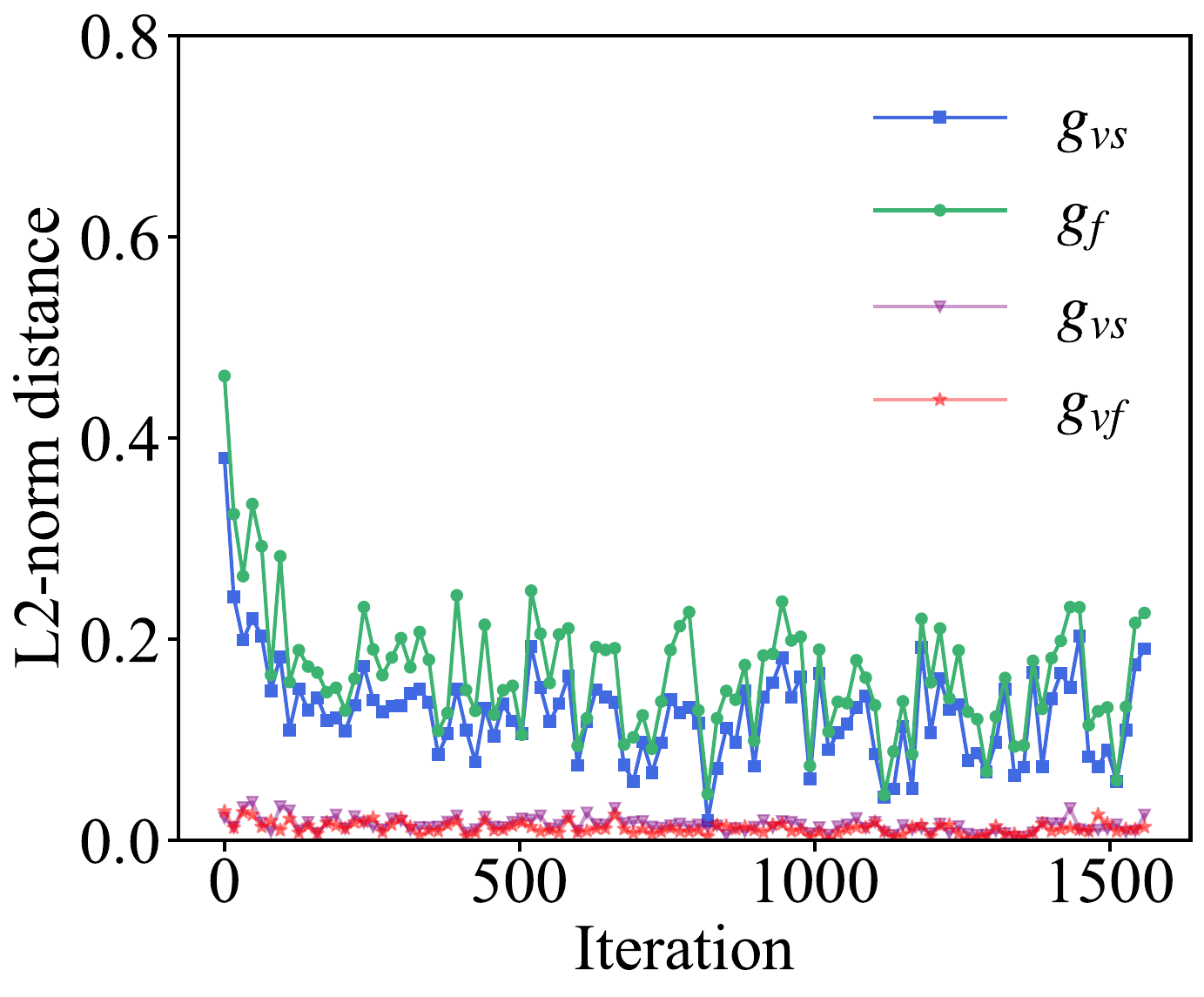}
        \caption{Task 7}
    \end{subfigure}
    \begin{subfigure}[b]{0.19\textwidth}
        \includegraphics[width=\textwidth]{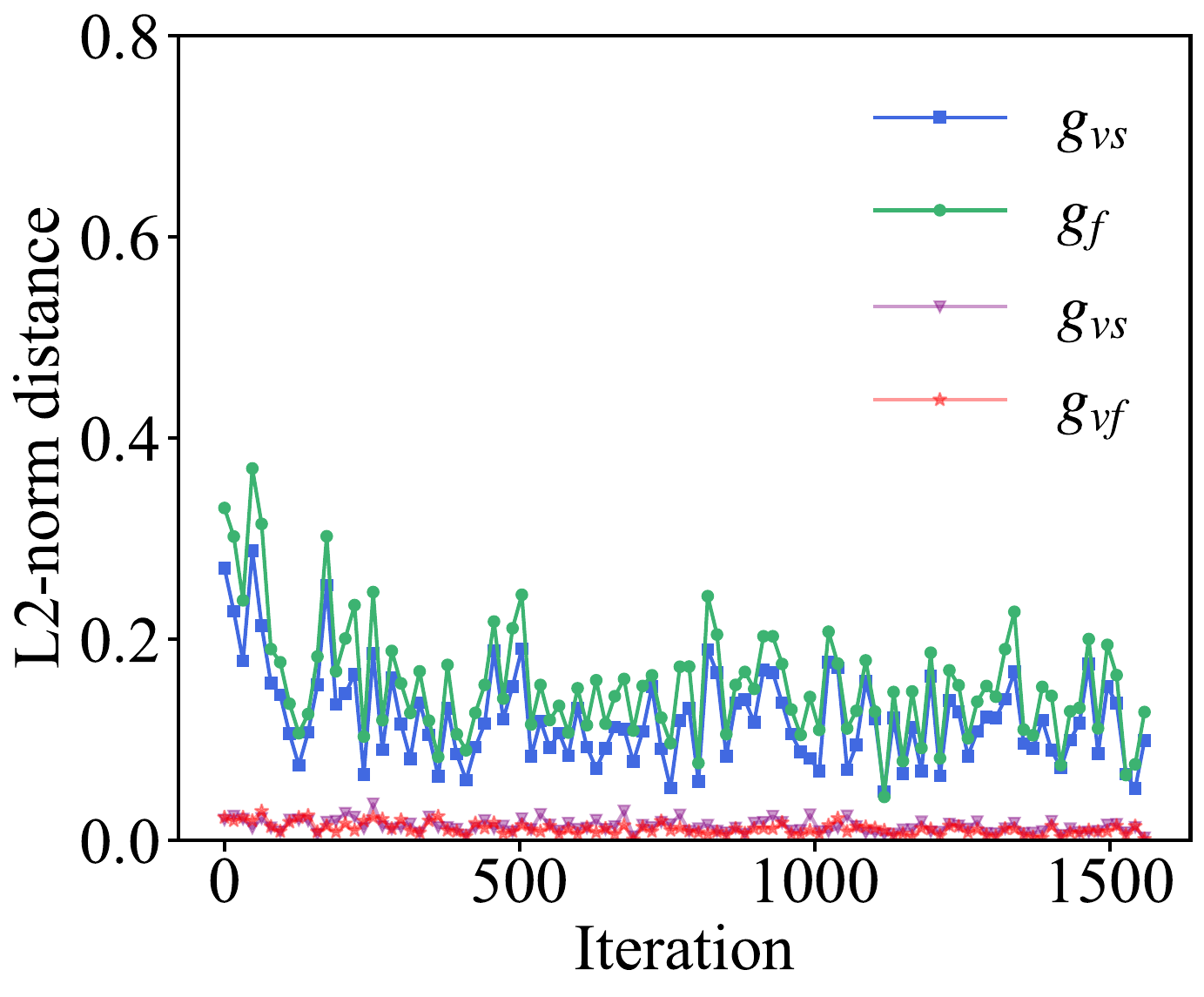}
        \caption{Task 8}
    \end{subfigure}
    \begin{subfigure}[b]{0.19\textwidth}
        \includegraphics[width=\textwidth]{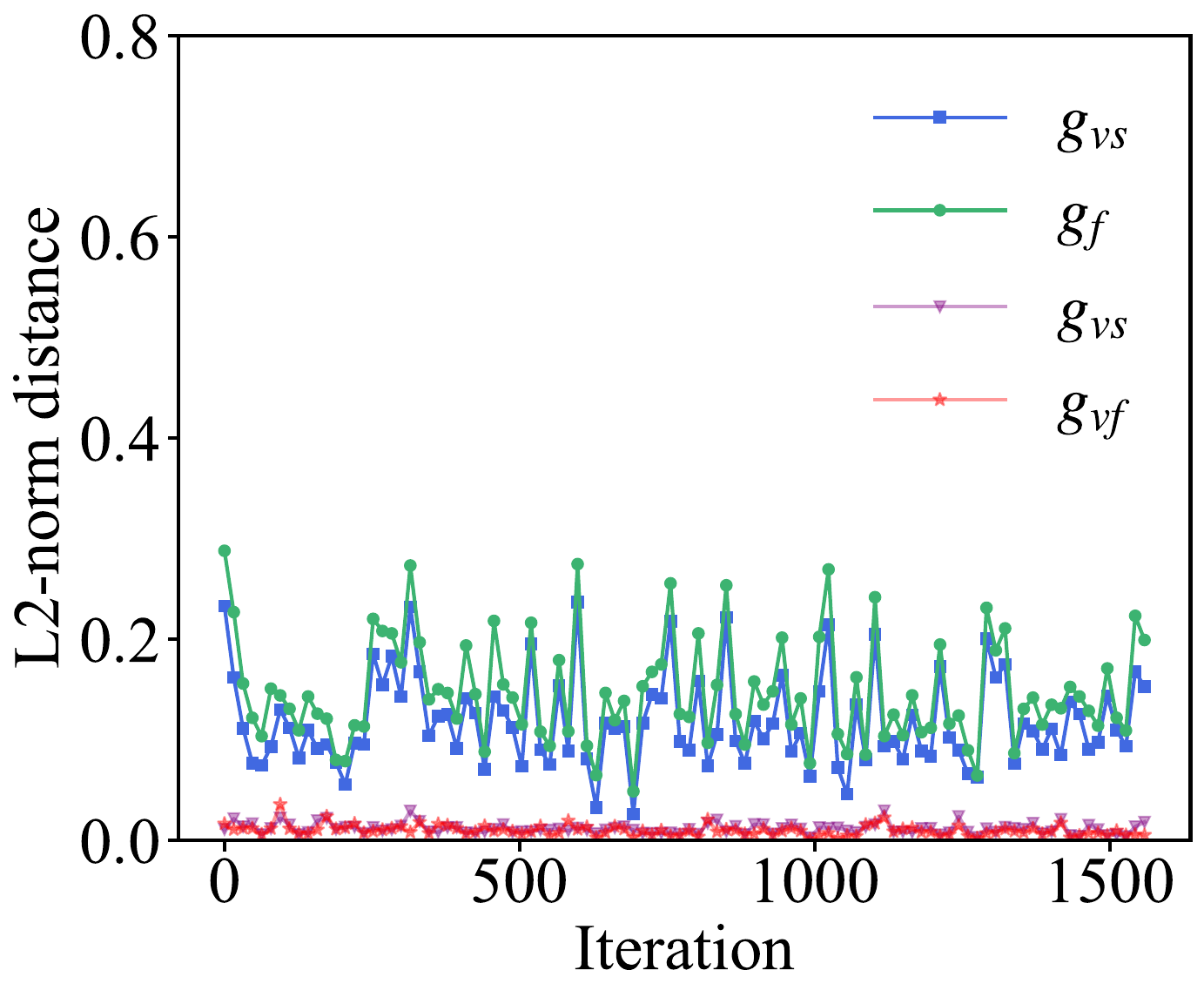}
        \caption{Task 9}
    \end{subfigure}
    \caption{Visualization of L2-norm distances of the gradients every 5 steps across 10 tasks.}
    \label{fig:diff_tasks}
\end{figure*}

\Cref{fig:diff_tasks} shows the L2-norm distances between sharpness and flatness gradients and their reference gradients across tasks. While $g$ and $g_0$ exhibit significant fluctuations during training, the gradients $g_{vs}$ and $g_{vf}$, core to zeroth-order sharpness and first-order flatness regularization, change much more slowly. This stability suggests their potential as shortcut directions for flat region exploration, bypassing the need for model ascent and backpropagation.



\subsection{Convergence of Turbo in the Surrogate Steps}
\label{sec:turbo-conv-appendix}

With the definitions in Appendix~\ref{sec:symbols}, we denote C-Flat gradient at iteration $t$ as
\begin{equation}
  \boldsymbol{G}_t
  := \nabla \mathcal{L}_{\mathrm{CFlat}}(\boldsymbol{\theta}_t)
  = \boldsymbol{g}_s + \lambda \boldsymbol{g}_f,
  \label{eq:Gt-def}
\end{equation}
then we can rewrite
\begin{equation}
  \boldsymbol{G}_t
  =
  \underbrace{\boldsymbol{g}}_{\text{empirical gradient}}
  +
  \underbrace{(\boldsymbol{g}_s - \boldsymbol{g})}_{\text{sharpness increment}}
  +
  \lambda \underbrace{\boldsymbol{g}_f}_{\text{flatness term}},
  \label{eq:Gt-decomp}
\end{equation}
where $\boldsymbol{g} = \nabla \mathcal{L}(\boldsymbol{\theta}_t)$ is the empirical gradient.

In the exact C-Flat updates, the sharpness increment $(\boldsymbol{g}_s - \boldsymbol{g})$
is obtained by subtracting the empirical gradient from the SAM gradient, and the flatness term
$\boldsymbol{g}_f$ is obtained via the finite-difference approximation
$\boldsymbol{g}_f \approx \frac{\rho}{\rho'}(\boldsymbol{g}_1 - \boldsymbol{g}_0)$
as in Eq.~\ref{eq:g-f-from-g0-g1}.

In Turbo, we do not recompute $\boldsymbol{g}_s$ and $\boldsymbol{g}_f$
at every surrogate step. Instead, we maintain EMA states
$\boldsymbol{g}_{vs,t}$ and $\boldsymbol{g}_{vf,t}$ to track
\[
  \boldsymbol{u}_t^{(s)} := \boldsymbol{g}_s - \boldsymbol{g},
  \qquad
  \boldsymbol{u}_t^{(f)} := \boldsymbol{g}_f,
\]
and we construct deterministic surrogates
\begin{align}
  A_t(\boldsymbol{g}_{vs,t-1}) &\approx \boldsymbol{u}_t^{(s)}, \\
  B_t(\boldsymbol{g}_{vf,t-1}) &\approx \boldsymbol{u}_t^{(f)}.
\end{align}

Therefore, at a Turbo surrogate step, the update direction is modeled as
\begin{equation}
  \tilde{\boldsymbol{g}}_t
  =
  \underbrace{\boldsymbol{g}}_{\nabla \mathcal{L}(\boldsymbol{\theta}_t)}
  + A_t(\boldsymbol{g}_{vs,t-1})
  + \lambda B_t(\boldsymbol{g}_{vf,t-1})
  + \boldsymbol{\xi}_t,
  \label{eq:gtilde-detailed}
\end{equation}
where $\boldsymbol{g} = \nabla \mathcal{L}(\boldsymbol{\theta}_t)$ is the empirical gradient and $\boldsymbol{\xi}_t$ is a zero-mean stochastic noise term capturing minibatch sampling and finite-difference randomness.

To relate Eq.~\ref{eq:gtilde-detailed} to the true C-Flat gradient
in Eq.~\ref{eq:Gt-decomp}, we decompose the deterministic surrogates as
\begin{align}
  A_t(\boldsymbol{g}_{vs,t-1})
  &= \boldsymbol{u}_t^{(s)} + \boldsymbol{\delta}_t^{(s)}, \\
  B_t(\boldsymbol{g}_{vf,t-1})
  &= \boldsymbol{u}_t^{(f)} + \boldsymbol{\delta}_t^{(f)},
\end{align}
where $\boldsymbol{u}_t^{(s)} := \boldsymbol{g}_s - \boldsymbol{g}$ and
$\boldsymbol{u}_t^{(f)} := \boldsymbol{g}_f$, and
$\boldsymbol{\delta}_t^{(s)}$, $\boldsymbol{\delta}_t^{(f)}$ are deterministic
approximation errors (given $\mathcal{F}_t$) coming from EMA tracking,
orthogonal decomposition residuals, and finite differences.

Substituting into Eq.~\ref{eq:gtilde-detailed} and comparing with Eq.~\ref{eq:Gt-decomp}, we obtain
\begin{align}
  \tilde{\boldsymbol{g}}_t
  &= \boldsymbol{g}
   + \bigl(\boldsymbol{u}_t^{(s)} + \boldsymbol{\delta}_t^{(s)}\bigr)
   + \lambda \bigl(\boldsymbol{u}_t^{(f)} + \boldsymbol{\delta}_t^{(f)}\bigr)
   + \boldsymbol{\xi}_t \\ \nonumber
  &= \underbrace{
       \bigl(\boldsymbol{g} + \boldsymbol{u}_t^{(s)}
       + \lambda \boldsymbol{u}_t^{(f)}\bigr)}_{=\boldsymbol{G}_t}
     + \underbrace{\bigl(\boldsymbol{\delta}_t^{(s)}
       + \lambda \boldsymbol{\delta}_t^{(f)}\bigr)}_{:=\boldsymbol{b}_t}
     + \boldsymbol{\xi}_t \\ \nonumber
  &= \boldsymbol{G}_t + \boldsymbol{b}_t + \boldsymbol{\xi}_t.
  \label{eq:gtilde-abstract}
\end{align}

\begin{assumption}[Generalized smoothness]
There exists $L>0$ such that for all
$\boldsymbol{\theta},\boldsymbol{\theta}'$
and any generalized gradients
$\boldsymbol{g} \in \partial \mathcal{L}_{\mathrm{CFlat}}(\boldsymbol{\theta})$,
$\boldsymbol{g}' \in \partial \mathcal{L}_{\mathrm{CFlat}}(\boldsymbol{\theta}')$,
\[
\|\boldsymbol{g} - \boldsymbol{g}'\|\le L\|\boldsymbol{\theta}-\boldsymbol{\theta}'\|.
\]
\end{assumption}

\begin{assumption}[Bounded gradients and noise]
There exist constants $G_{\max},\Sigma\ge0$ such that for all $t$
and any $\boldsymbol{g} \in \partial \mathcal{L}_{\mathrm{CFlat}}(\boldsymbol{\theta}_t)$,
\[
\|\boldsymbol{g}\|\le G_{\max},\qquad
\mathbb{E}[\boldsymbol{\xi}_t\mid\mathcal{F}_t]=\boldsymbol{0},\qquad
\mathbb{E}[\|\boldsymbol{\xi}_t\|^2\mid\mathcal{F}_t]\le \Sigma^2.
\]
\end{assumption}

\begin{assumption}[Deterministic approximation bias]
The bias $\boldsymbol{b}_t$ satisfies $\|\boldsymbol{b}_t\|\le \Delta_t$
almost surely, where $\Delta_t\ge0$ is $\mathcal{F}_t$-measurable.
In Turbo, $\Delta_t$ aggregates:
\begin{enumerate}
  \item EMA tracking error in estimating $\boldsymbol{g}_s-\boldsymbol{g}$ and
        $\boldsymbol{g}_f$ via $\boldsymbol{g}_{vs,t}$ and $\boldsymbol{g}_{vf,t}$;
  \item orthogonal decomposition residuals in constructing the direction-invariant
        sharpness and flatness component $\boldsymbol{g}_{vs}$ and $\boldsymbol{g}_{vf}$;
  \item finite-difference approximation errors in the proxy and proxy-perturbed
        gradients $\boldsymbol{g}_0,\boldsymbol{g}_1$.
\end{enumerate}
\end{assumption}

\begin{assumption}[Stepsize and perturbation schedule]
We use the schedule in~\cite{zhang2023gradient}:
\[
\eta_t = \frac{\eta_0}{\sqrt{t}},\qquad
\rho_t=\frac{\rho_0}{\sqrt{t}},
\]
with $\eta_0$ chosen so that $\eta_t \le 1/(4L)$ for all $t$.
\end{assumption}

Finally, the EMA tracking of $\boldsymbol{u}_t^{(s)}$ and
$\boldsymbol{u}_t^{(f)}$ yields a bias sequence $\Delta_t$ satisfying
\begin{equation}
  \sum_{t=1}^T \eta_t \Delta_t^2
  \;\le\; C_{\Delta,1}\sqrt{T} + C_{\Delta,2}\log T
  \label{eq:Delta-sum-assumption}
\end{equation}
for some constants $C_{\Delta,1},C_{\Delta,2} \ge 0$.

\paragraph{Justification of the deterministic bias assumption.}
The weighted bound on $\Delta_t$ in Eq.~\ref{eq:Delta-sum-assumption} is supported by the empirical directional stability of the cached components. As illustrated in Figure~\ref{fig:grad_var} and~\ref{fig:diff_tasks} in the main text, the direction invariant components $\boldsymbol{g}_{vs,t}$ and $\boldsymbol{g}_{vf,t}$ exhibit substantially higher cosine similarity
across consecutive iterations than the empirical gradient $\boldsymbol{g}_t$ and the proxy gradient $\boldsymbol{g}_{0,t}$. This high stability suggests that the surrogate approximations
$A_t(\cdot)$ and $B_t(\cdot)$ remain accurate within Turbo intervals (i.e., $\|\boldsymbol{\delta}_t^{(s)}\|$ and $\|\boldsymbol{\delta}_t^{(f)}\|$ stay small), which in turn justifies modeling the resulting approximation bias $\Delta_t$ controlled.

\begin{lemma}[Generalized descent lemma]
Under the generalized smoothness assumption, for any update
$\boldsymbol{\theta}_{t+1}=\boldsymbol{\theta}_t - \eta_t \tilde{\boldsymbol{g}}_t$
and any $\boldsymbol{G}_t \in \partial \mathcal{L}_{\mathrm{CFlat}}(\boldsymbol{\theta}_t)$,
\[
\mathcal{L}_{\mathrm{CFlat}}(\boldsymbol{\theta}_{t+1}) \le
\mathcal{L}_{\mathrm{CFlat}}(\boldsymbol{\theta}_t)
- \eta_t \langle \boldsymbol{G}_t,\tilde{\boldsymbol{g}}_t\rangle
+ \tfrac{L}{2}\eta_t^2 \|\tilde{\boldsymbol{g}}_t\|^2.
\]
\end{lemma}

\begin{proof}
By generalized smoothness and the descent inequality for $L$-smooth
functions, we have
\[
\mathcal{L}_{\mathrm{CFlat}}(\boldsymbol{\theta}_{t+1})
\le \mathcal{L}_{\mathrm{CFlat}}(\boldsymbol{\theta}_t)
+ \langle \boldsymbol{G}_t,\boldsymbol{\theta}_{t+1}-\boldsymbol{\theta}_t\rangle
+ \tfrac{L}{2}\|\boldsymbol{\theta}_{t+1}-\boldsymbol{\theta}_t\|^2.
\]
Substitute $\boldsymbol{\theta}_{t+1}-\boldsymbol{\theta}_t
= -\eta_t\tilde{\boldsymbol{g}}_t$ and rearrange.
\end{proof}

For EMA tracking, the bias sequence induced by
$\boldsymbol{g}_{vs,t}$ and $\boldsymbol{g}_{vf,t}$ is absorbed into
$\Delta_t$, and we simply assume that it satisfies the bound in
Eq.~\ref{eq:Delta-sum-assumption}, which is standard for exponential
moving averages under mild regularity conditions.

\paragraph{Single-step expected decrease.}
Starting from the descent lemma and using
$\tilde{\boldsymbol{g}}_t = \boldsymbol{G}_t + \boldsymbol{b}_t + \boldsymbol{\xi}_t$,
we obtain
\begin{align}
\mathcal{L}_{\mathrm{CFlat}}(\boldsymbol{\theta}_{t+1})
&\le \mathcal{L}_{\mathrm{CFlat}}(\boldsymbol{\theta}_t)
- \eta_t \langle \boldsymbol{G}_t,\boldsymbol{G}_t+\boldsymbol{b}_t+\boldsymbol{\xi}_t\rangle \\ \nonumber
&\quad+ \tfrac{L}{2}\eta_t^2 \|\boldsymbol{G}_t+\boldsymbol{b}_t+\boldsymbol{\xi}_t\|^2.
\end{align}
Taking full expectation, and using
$\mathbb{E}[\langle\boldsymbol{G}_t,\boldsymbol{\xi}_t\rangle\mid\mathcal{F}_t]=0$, we get
\begin{align}
\label{eq:one-step-start}
\mathbb{E}[\mathcal{L}_{\mathrm{CFlat}}(\boldsymbol{\theta}_{t+1})]
&\le \mathbb{E}[\mathcal{L}_{\mathrm{CFlat}}(\boldsymbol{\theta}_t)]
- \eta_t \mathbb{E}[\|\boldsymbol{G}_t\|^2] \\ \nonumber
&\quad- \eta_t \mathbb{E}[\langle \boldsymbol{G}_t,\boldsymbol{b}_t\rangle]
+ \tfrac{L}{2}\eta_t^2 \mathbb{E}[\|\boldsymbol{G}_t+\boldsymbol{b}_t+\boldsymbol{\xi}_t\|^2].
\end{align}

\textbf{a) Bias inner product term:}
By Cauchy--Schwarz and $\|\boldsymbol{b}_t\|\le\Delta_t$, we have
\[
|\langle \boldsymbol{G}_t,\boldsymbol{b}_t\rangle|
\le \|\boldsymbol{G}_t\|\,\|\boldsymbol{b}_t\|
\le \tfrac{1}{2}\|\boldsymbol{G}_t\|^2 + \tfrac{1}{2}\Delta_t^2.
\]
Using $|\mathbb{E}[Z]|\le \mathbb{E}[|Z|]$, this implies
\begin{equation}
-\eta_t \mathbb{E}[\langle \boldsymbol{G}_t,\boldsymbol{b}_t\rangle]
\le \tfrac{\eta_t}{2}\mathbb{E}\|\boldsymbol{G}_t\|^2
+ \tfrac{\eta_t}{2}\Delta_t^2.
\label{eq:bias-inner}
\end{equation}

\textbf{b) Second moment term:}
For any three vectors $x,y,z$, Jensen's inequality gives
\[
\bigl\|x+y+z\bigr\|^2
\le 3\bigl(\|x\|^2 + \|y\|^2 + \|z\|^2\bigr).
\]
Applying this with $x=\boldsymbol{G}_t$, $y=\boldsymbol{b}_t$ and $z=\boldsymbol{\xi}_t$, we obtain
\[
\|\boldsymbol{G}_t + \boldsymbol{b}_t + \boldsymbol{\xi}_t\|^2
\le 3\|\boldsymbol{G}_t\|^2 + 3\|\boldsymbol{b}_t\|^2 + 3\|\boldsymbol{\xi}_t\|^2,
\]
hence, using $\|\boldsymbol{b}_t\|\le\Delta_t$ and
$\mathbb{E}[\|\boldsymbol{\xi}_t\|^2\mid\mathcal{F}_t]\le\Sigma^2$,
\begin{align}
\mathbb{E}[\|\boldsymbol{G}_t+\boldsymbol{b}_t+\boldsymbol{\xi}_t\|^2]
&\le 3\,\mathbb{E}\|\boldsymbol{G}_t\|^2 + 3\Delta_t^2 + 3\Sigma^2.
\label{eq:second-moment}
\end{align}

Substituting Eq.~\ref{eq:bias-inner} and Eq.~\ref{eq:second-moment} into
Eq.~\ref{eq:one-step-start}, we obtain
\begin{align}
\mathbb{E}[\mathcal{L}_{\mathrm{CFlat}}(\boldsymbol{\theta}_{t+1})]
&\le \mathbb{E}[\mathcal{L}_{\mathrm{CFlat}}(\boldsymbol{\theta}_t)]
- \eta_t \mathbb{E}\|\boldsymbol{G}_t\|^2
+ \tfrac{\eta_t}{2}\mathbb{E}\|\boldsymbol{G}_t\|^2 \notag\\
&\quad
+ \tfrac{\eta_t}{2}\Delta_t^2 + \tfrac{L}{2}\eta_t^2 \bigl(3\mathbb{E}\|\boldsymbol{G}_t\|^2 + 3\Delta_t^2 + 3\Sigma^2\bigr) \notag\\
&= \mathbb{E}[\mathcal{L}_{\mathrm{CFlat}}(\boldsymbol{\theta}_t)]
- \Bigl(\tfrac{\eta_t}{2} - \tfrac{3L}{2}\eta_t^2\Bigr)\mathbb{E}\|\boldsymbol{G}_t\|^2 \notag\\
&\quad
+ \Bigl(\tfrac{\eta_t}{2} + \tfrac{3L}{2}\eta_t^2\Bigr)\Delta_t^2
+ \tfrac{3L}{2}\eta_t^2\Sigma^2.
\end{align}
Under the stepsize condition $\eta_t \le 1/(4L)$, we have
\[
\tfrac{\eta_t}{2} - \tfrac{3L}{2}\eta_t^2
\;\ge\; \tfrac{\eta_t}{8},
\]
and hence
\[
\frac{\eta_t}{8}\,\mathbb{E}\|\boldsymbol{G}_t\|^2
\le \mathbb{E}[\mathcal{L}_{\mathrm{CFlat}}(\boldsymbol{\theta}_t)]
- \mathbb{E}[\mathcal{L}_{\mathrm{CFlat}}(\boldsymbol{\theta}_{t+1})]
+ \Psi_t,
\]
where
\[
\Psi_t
:= \Bigl(\tfrac{\eta_t}{2} + \tfrac{3L}{2}\eta_t^2\Bigr)\Delta_t^2
 + \tfrac{3L}{2}\eta_t^2\Sigma^2.
\]

\paragraph{Summation.}
Summing from $t=1$ to $T$, we obtain
\[
\sum_{t=1}^T \frac{\eta_t}{8}\mathbb{E}\|\boldsymbol{G}_t\|^2
\le \mathcal{L}_{\mathrm{CFlat}}(\boldsymbol{\theta}_1) - L^* + \sum_{t=1}^T \Psi_t,
\]
where $L^* = \inf_{\boldsymbol{\theta}}\mathcal{L}_{\mathrm{CFlat}}(\boldsymbol{\theta})$.
Hence
\begin{equation}
\sum_{t=1}^T \eta_t \mathbb{E}\|\boldsymbol{G}_t\|^2
\le 8\bigl(\mathcal{L}_{\mathrm{CFlat}}(\boldsymbol{\theta}_1)-L^*\bigr)
+ 8\sum_{t=1}^T \Psi_t.
\label{eq:sum-grad}
\end{equation}

Using $\eta_t=\eta_0/\sqrt{t}$ and the assumption
$\sum_{t=1}^T \eta_t\Delta_t^2 \le C_{\Delta,1}\sqrt{T} + C_{\Delta,2}\log T$,
together with $\sum_{t=1}^T \eta_t^2 \le \eta_0^2(1+\log T)$, we obtain
\[
\sum_{t=1}^T \Psi_t
\le C_1\sqrt{T} + C_2\log T + C_3
\]
for some constants $C_1,C_2,C_3$.
Thus there exist $C_A,C_B,C_C$ such that
\[
\sum_{t=1}^T \eta_t \mathbb{E}\|\boldsymbol{G}_t\|^2
\le C_A + C_B\sqrt{T} + C_C\log T.
\]

Since $\eta_T = \eta_0/\sqrt{T}$ and $\eta_t$ is nonincreasing, we have
\[
\eta_T \sum_{t=1}^T \mathbb{E}\|\boldsymbol{G}_t\|^2
\le \sum_{t=1}^T \eta_t \mathbb{E}\|\boldsymbol{G}_t\|^2,
\]
so
\[
\frac{\eta_0}{\sqrt{T}}\sum_{t=1}^T \mathbb{E}\|\boldsymbol{G}_t\|^2
\le C_A + C_B\sqrt{T} + C_C\log T.
\]
Dividing by $T$ yields
\[
\frac{1}{T}\sum_{t=1}^T \mathbb{E}\|\boldsymbol{G}_t\|^2
\le \frac{C_A}{\eta_0\sqrt{T}}
 + \frac{C_B}{\eta_0}
 + \frac{C_C\log T}{\eta_0\sqrt{T}}.
\]

\begin{theorem}[Turbo convergence in surrogate steps]
Under the above assumptions and the EMA tracking condition Eq.~\ref{eq:Delta-sum-assumption}, there exist constants $\tilde C_1,\tilde C_2,\tilde C_3 \ge 0$ such that
\[
\frac{1}{T}\sum_{t=1}^T
\mathbb{E}\big[\|\nabla \mathcal{L}_{\mathrm{CFlat}}(\boldsymbol{\theta}_t)\|^2\big]
\le \frac{\tilde C_1 + \tilde C_2\log T}{\sqrt{T}} + \tilde C_3.
\]
Moreover, if the bias sequence satisfies $\Delta_t = O(1/\sqrt{t})$ so that $\sum_{t=1}^T \eta_t\Delta_t^2 = O(1+\log T)$, then $\tilde C_3 = 0$ and Turbo recovers a GAM-style rate~\cite{bian2024make}:
\[
\frac{1}{T}\sum_{t=1}^T
\mathbb{E}\big[\|\nabla \mathcal{L}_{\mathrm{CFlat}}(\boldsymbol{\theta}_t)\|^2\big]
= O\!\left(\frac{1+\log T}{\sqrt{T}}\right).
\]
\end{theorem}

\subsection{C-Flat Turbo Algorithm}
\enlargethispage{3\baselineskip}
\begin{algorithm}[h]
\small
\caption{C-Flat Turbo}\label{algo:opt}
\begin{algorithmic}[1]
\Statex \textbf{Input:} Training phase $T$, training data $S^T$, model parameter $\boldsymbol{\theta}$, total iterations $J$, oracle loss function $\mathcal{L}$, learning rate $\eta$, C-Flat coefficient $\lambda$, Turbo step $k$, $\mu_{s,0} = \mu_{f, 0} = 0$, $\sigma_{s,0} = \sigma_{f,0} = 10^{-8}$.
\Statex \textbf{Output:} Model trained at the current time $T$ with Turbo.

\For{$j = 1$ to $J$, sample batch $B^T_j$ from dataset $S^T$}
    \State Compute empirical gradient: $\boldsymbol{g} = \nabla \mathcal{L}(\boldsymbol{\theta})$
    \State Initialize update direction: $\overline{\boldsymbol{g}} = \boldsymbol{g}$
    \State Update EMA statistics: $\mu_{s,j}, \sigma_{s,j}, \mu_{f,j}, \sigma_{f,j}$ by Eq.~\ref{eq:ema_f}

    \If{$\|\boldsymbol{g}\|^2 \geq \mu_{s,j} + \sigma_{s,j}$}
        \If{$j \bmod k = 0$}
            \State Compute SAM gradient $\boldsymbol{g}_s = \nabla \mathcal{L}(\boldsymbol{\theta} + \boldsymbol{\epsilon}_{0}^*)$ by Eq.~\ref{eq:grad_s}
            \State Cache direction-invariant sharpness component $\boldsymbol{g}_{vs}$ by Eq.~\ref{eq:gvs}
        \Else
            \State Simulate sharpness increment using cached $\boldsymbol{g}_{vs}$, that means $\boldsymbol{g}_s
              = \boldsymbol{g}
              + \beta \frac{\|\boldsymbol{g}\|}{\|\boldsymbol{g}_{vs}\|}\,\boldsymbol{g}_{vs}$
        \EndIf
        \State Update direction: $\overline{\boldsymbol{g}} = \boldsymbol{g}_s$
    \EndIf  

    \State Compute proxy model parameter $\boldsymbol{\theta}_p = \boldsymbol{\theta} + \boldsymbol{\epsilon}_1^*$ by Eq.~\ref{eq:grad_f}
    \State Compute proxy gradient: $\boldsymbol{g}_0 = \nabla \mathcal{L}(\boldsymbol{\theta}_p)$

    \If{$\|\boldsymbol{g}_0\|^2 \geq \mu_{f,j} + \sigma_{f,j}$}
        \If{$j \bmod k = 0$}
            \State Compute proxy-perturbed gradient $\boldsymbol{g}_1 = \nabla \mathcal{L}\!\left(\boldsymbol{\theta}_p + \rho' \frac{\boldsymbol{g}_0}{\|\boldsymbol{g}_0\|}\right)$
            \State Compute first-order flatness gradient $\boldsymbol{g}_f$ from $\boldsymbol{g}_0,\boldsymbol{g}_1$ by Eq.~\ref{eq:grad_f}
            \State Cache direction-invariant flatness component $\boldsymbol{g}_{vf}$ by Eq.~\ref{eq:gvf}
        \Else
            \State Simulate flatness direction using cached $\boldsymbol{g}_{vf}$, that means $\boldsymbol{g}_f
              = \boldsymbol{g}_0
              + \beta \frac{\|\boldsymbol{g}_0\|}{\|\boldsymbol{g}_{vf}\|}\,\boldsymbol{g}_{vf}$
        \EndIf  
        \State Update direction: $\overline{\boldsymbol{g}} = \overline{\boldsymbol{g}} + \lambda \cdot \boldsymbol{g}_f$
    \EndIf  

    \State Optionally update the Turbo step $k$ according to Section~\ref{subsec:sche}
    \State Update model parameter: $\boldsymbol{\theta}^{T} = \boldsymbol{\theta}^{T} - \eta \cdot \overline{\boldsymbol{g}}$
\EndFor

\end{algorithmic}
\end{algorithm}

\end{document}